\documentclass[nohyperref]{article}

\usepackage{microtype}
\usepackage{graphicx}
\usepackage{subfigure}
\usepackage{booktabs} % for professional tables

\usepackage{hyperref}
\usepackage[accepted]{icml2022}

\usepackage{amsmath}
\usepackage{amssymb}
\usepackage{mathtools}
\usepackage{amsthm}

\usepackage[capitalize,noabbrev]{cleveref}

\usepackage[textsize=tiny]{todonotes}

\usepackage{amssymb}
\usepackage{amsthm}
\usepackage{amsfonts}
\usepackage{dsfont}
\usepackage{mathtools}
\allowdisplaybreaks

\usepackage{diagbox}
\usepackage{multirow}
\usepackage{graphicx}
\usepackage{epsfig}
\usepackage{float}
\usepackage{wrapfig}

\usepackage{mdframed}
\usepackage{xcolor}

\usepackage{appendix}
\usepackage{comment}
\usepackage{paralist}
\usepackage{enumitem}
\usepackage{pdfpages}
\newtheorem{theorem}{Theorem}{\itshape}{\rmfamily}

\newtheorem*{definition}{Definition}{\itshape}{\rmfamily}

\newtheorem{lemma}{Lemma}[theorem]{\itshape}{\rmfamily} 	% let lemmas use sub-level numbers under theorem

\newtheorem{proposition}[theorem]{Proposition}{\itshape}{\rmfamily}

{\itshape}{\rmfamily}

{\itshape}{\rmfamily}

{\itshape}{\rmfamily}

\newtheorem*{recap}{}{\itshape}{\rmfamily}

\newenvironment{proofidea}{\par{\noindent \emph{Proof idea:}}}{\qed\par}

\newcommand{\squishlist}{
	\begin{list}{$\bullet$}
		{ 	\setlength{\itemsep}{0pt}      \setlength{\parsep}{1pt}
			\setlength{\topsep}{3pt}       \setlength{\partopsep}{0pt}
			\setlength{\leftmargin}{1.5em} \setlength{\labelwidth}{1em}
			\setlength{\labelsep}{0.5em} } }

\newcommand{\squishlisttwo}{
	\begin{list}%{$\bullet$}
		{ \setlength{\itemsep}{0pt}    \setlength{\parsep}{0pt}
			\setlength{\topsep}{0pt}     \setlength{\partopsep}{0pt}
			\setlength{\leftmargin}{2em} \setlength{\labelwidth}{1.5em}
			\setlength{\labelsep}{0.5em} } }
	
\newcommand{\squishend}{\end{list}}

\newcommand{\eq}[2][]{
	\ifx\empty#1\empty
		\begin{equation*} #2 \end{equation*}
	\else
		\begin{equation}\label{#1} #2 \end{equation}
	\fi
}
\newcommand{\eqm}[2][]{
	\ifx\empty#1\empty
		\begin{equation*}\begin{split} #2 \end{split}\end{equation*}
	\else
		\begin{equation}\label{#1}\begin{split} #2 \end{split}\end{equation}
	\fi
}
\newcommand{\eqa}[2][]{
	\ifx\empty#1\empty
		\begin{align*} #2 \end{align*}
	\else
		\begin{align}\label{#1} #2 \end{align}
	\fi
}

\newcommand{\define}[0]{\doteq}

\newcommand{\ra}[0]{\rightarrow}

\newcommand{\Op}[2]{#1 \Big[ #2 \Big]}
\renewcommand{\P}{\mathbf{P}}
\renewcommand{\Pr}[2][]{\P_{#1}[#2]}

\newcommand{\E}[1][]{\ifx\empty#1\empty \mathbf{E} \else \operatorname*{\mathbf{E}}_{\substack{#1}} \fi}
\newcommand{\Exp}[2][]{\Op{\E[#1]~}{#2}}

\newcommand{\indicator}[1]{\mathds{1}[ #1 ]}
\newcommand{\Real}[0]{\mathbb{R}}

\renewcommand{\*}{\cdot}
\newcommand{\vect}[1]{\boldsymbol{#1}}

\newcommand{\AS}{\mathcal{A}}
\renewcommand{\SS}{\mathcal{S}}

\newcommand{\T}{P}
\newcommand{\R}{R}
\newcommand{\ga}{\gamma}

\newcommand{\w}{\vect{w}}

\newcommand{\terminals}{\SS_\bot}

\newcommand{\vocab}[1][]{\Sigma_{#1}}
\newcommand{\seq}[1]{(#1)}
\newcommand{\X}[1][]{\ifx\empty#1\empty X \else X^{\scriptscriptstyle (#1)} \fi}
\newcommand{\Y}[1][]{\ifx\empty#1\empty Y \else Y^{\scriptscriptstyle (#1)} \fi}
\newcommand{\Z}[1][]{\ifx\empty#1\empty Z \else Z^{\scriptscriptstyle (#1)} \fi}
\renewcommand{\L}[1][]{\ifx\empty#1\empty L \else L^{\scriptscriptstyle (#1)} \fi}
\newcommand{\Yp}[2][]{\ifx\empty#1\empty Y_{<#2} \else Y^{\scriptscriptstyle (#1)}_{<#2} \fi}
\newcommand{\Zp}[2][]{\ifx\empty#1\empty Z_{<#2} \else Z^{\scriptscriptstyle (#1)}_{<#2} \fi}
\newcommand{\y}[2][]{\ifx\empty#1\empty y_{\scriptscriptstyle #2} \else y^{\scriptscriptstyle (#1)}_{\scriptscriptstyle #2} \fi}
\newcommand{\z}[2][]{\ifx\empty#1\empty z_{\scriptscriptstyle #2} \else z^{\scriptscriptstyle (#1)}_{\scriptscriptstyle #2} \fi}
\newcommand{\bos}{\texttt{bos}}
\newcommand{\eos}{\texttt{eos}}

\newcommand{\B}{{ \mathcal{B} }}
\newcommand{\Bopt}[1][]{{\ifx\empty#1\empty \B \else \B^{#1} \fi}}
\newcommand{\Qopt}{{ Q_{\max} }}

\newcommand{\Lagr}{{ \mathcal{L} }}
\newcommand{\mul}{{ \vect{\lambda} }}
\newcommand{\Q}{{ \mathcal{Q} }}
\newcommand{\gammaEPI}{{ \gamma_\text{~epi~} }}

\newcommand{\Qmin}{Q_{\min}}
\newcommand{\Qmax}{Q_{\max}}

\icmltitlerunning{Lagrangian Method for Q-Function Learning}

\begin{document}

\twocolumn[
\icmltitle{Lagrangian Method for Q-Function Learning \\(with Applications to Machine Translation)}



\icmlsetsymbol{equal}{*}

\begin{icmlauthorlist}
\icmlauthor{Huang Bojun}{rit}
\end{icmlauthorlist}

\icmlaffiliation{rit}{Rakuten Institute of Technology, Rakuten Group Inc., Japan}
\icmlcorrespondingauthor{Huang Bojun}{bojhuang@gmail.com}

\icmlkeywords{Q-function Learning, Lagrangian Duality, Learning Theory, Machine Translation}

\vskip 0.3in
]



\printAffiliationsAndNotice{}  % leave blank if no need to mention equal contribution

\begin{abstract}
This paper discusses a new approach to the fundamental problem of learning optimal Q-functions. In this approach, optimal Q-functions are formulated as saddle points of a \emph{nonlinear} Lagrangian function derived from the classic Bellman optimality equation. The paper shows that the Lagrangian enjoys \emph{strong duality}, in spite of its nonlinearity, which paves the way to a general Lagrangian method to Q-function learning. As a demonstration, the paper develops an imitation learning algorithm based on the duality theory, and applies the algorithm to a state-of-the-art machine translation benchmark. The paper then turns to demonstrate a symmetry breaking phenomenon regarding the optimality of the Lagrangian saddle points, which justifies a largely overlooked direction in developing the Lagrangian method.  
\end{abstract}

\section{Introduction}
\label{sec:introduction}

Machine learning methods can be broadly categorized into two classes: policy learning, and Q-function learning.
In policy learning, the goal is to learn a policy function that directly encodes the desired mapping between task inputs and outputs. An imitation learning example is \emph{behavior cloning}, which tunes a parametric policy toward the desired policy function based on example input-output pairs of the latter. A reinforcement learning example is \emph{policy gradient}, which iteratively improves the parametric policy function following the gradient direction of the task performance.

Alternatively, one can try to learn a \emph{Q-function} that encodes a preferential score for each candidate output given an input. An \emph{optimal Q-function} enables a \emph{greedy policy} to identify the optimal output by simply comparing the Q-values, without looking at the precise consequences of the outputs. 

Implicitly or explicitly, the majority of statistical learning practices nowadays fall in the category of Q-function learning, where the general strategy is to make the Q-value of each output equal to the probability that an expert would choose this output (conditioned on the input).\footnote{
	Note that although the learned Q-function represents a desired probability distribution, this distribution is not used to sample the outputs at decision time (which contrasts with behavior cloning), but is used to power a greedy and deterministic decision policy.
} The resulted Q-greedy policy is provably optimal in many simple yet common scenarios, such as for maximizing prediction accuracy in classification tasks~\cite{2006:prml}.

In more complex scenarios, such as in \emph{sequential decision making} or \emph{structured prediction} tasks, Q-functions that simply follow the expert probabilities are generally not optimal, and learning optimal Q-functions becomes a more challenging problem. In this case, the standard paradigm at the moment is to learn the \emph{Bellman optimal value} function~\cite{2018:RL}, which is a special optimal Q-function that complies with a particular constraint known as the \emph{Bellman equation}. The Bellman equation also immediately implies a \emph{value iteration} method to find/approximate the optimal Q-function that the equation defined, leading to the many variants of \emph{Q-Learning} algorithms~\cite{1989:qlearning}. 

In this paper, we propose a new method for optimal Q-function learning. Our method is based on a variational (re-)formulation of the Bellman equation. The variational formulation induces a nonlinear \emph{Lagrangian function} whose saddle points correspond to a class of optimal Q-functions (but not necessarily being the Bellman optimal value). The Q-function learning problem can then be reduced to saddle-point optimization for the Lagrangian. 

This high-level idea is known as the \emph{Lagrangian method} in a wide variety of engineering domains~\cite{2004:convex}, including machine learning~\cite{2014:gan}, and more specifically, has been recently applied to related value-function learning problems~\cite{2018:LP_RL, 2018:LP_qlearning, 2019:dualDICE, 2020:bestDICE}. There is however a key obstacle to apply the Lagrangian method to Q-function learning: To admit effective and efficient algorithms, it is typically required that the Lagrangian function has a \emph{strong duality} property. When the Lagrangian is linear, the strong duality property is universally guaranteed, which is the theoretical foundation for most of its algorithmic applications. Unfortunately, for (optimal) Q-function learning, the Lagrangian derived from the Bellman equation is nonlinear, which prevents further algorithmic developments.

Our first main contribution is to prove that, for a large class of decision tasks, the nonlinear Lagrangian function for optimal Q-function learning turns out to be a special nonlinear function that actually enjoys the strong duality ``as usual'' (Section \ref{sec:lagrangian}). This observation potentially opens the door to a new general approach to Q-function learning.

As a demonstration of this potential, we developed a simple imitation learning algorithm based on the Lagrangian duality theory. We presented two practical versions of the algorithm, and empirically applied them to Machine Translation (MT) as a case study. Our algorithms are able to train Transformer~\cite{2017:transformer}, a state-of-the-art neural network model, on large-scale MT benchmarks with real-world data, and lead to $1.4$ BLEU (=$5\%$) improvement over standard baseline (Section \ref{sec:mt}). 

Thirdly, we discovered an unusual \emph{symmetry breaking} phenomenon for the Lagrangian of optimal Q-functions: From the same Bellman equation, there can be two mirrored ways to derive the same Lagrangian function, resulting in two kinds of saddle-point solutions, minimax points and maximin points. While previous research on the topic has been exclusively focusing on the former (i.e. minimax points), they can be sub-optimal in modern learning settings, as we show in Section 5. Intriguingly, we prove that the other class of saddle points, the maximin points, turn out to guarantee optimality. This observation points to a new direction that has perhaps been overlooked for a decades-old topic. 

Lastly, our paper provides a new and more general theorem about the Bellman optimality structure (see Theorem \ref{thm:Qstar} in Section \ref{sec:preliminaries}), which enabled us to develop Q-learning theory on top of the (episode-wise) \emph{total reward} optimality criterion. This criterion is a more accurate formulation than the canonical discounted-MDP setting for many real-world applications. We believe that our mathematical treatment contributes to fill this well-known gap between theoretical formulation and practical setup for Q-function learning.

\section{Problem Formulation}
\label{sec:preliminaries}

In general, a decision task is mathematically a Markov Decision Process (MDP) $(\SS, \AS, \R, \T, \rho)$. $\SS$ is the state space, $\AS$ the action space. $\R(s)$ is a bounded reward (possibly negative) associated to each state $s\in \SS$.\footnote{
	Our state-based reward formulation follows \cite{2015:trpo} and \cite{2020:bojun}, and is equivalent to, if not more general than, the various action-based reward formulations (that some readers might feel more familiar with). See Appendix \ref{sec:more_about_mdp} for details.
}
$\T(s'|s,a)$ specifies action-conditioned transition probabilities between states, and $\rho$ is the initial-state distribution.
Given an MDP, a policy function $\pi$ specifies the action selection probabilities under each state, which induces a Markov chain with $\Pr[\pi]{S_1=s} = \rho(s)$ and $\Pr[\pi]{S_{t+1}=s'|S_t=s} = \sum_{a\in\AS} \T(s'|s,a) \* \pi(a|s)$. 
Running $\pi$ in the Markov chain $\P_{\pi}$ 
%Running $\pi$ in the MDP
results in an infinite trajectory $\zeta=(S_1,A_1,S_2,A_2\dots)$.

In this paper, we focus on \emph{finite-time} decision tasks, in which there is a set of \textbf{terminal states} $\terminals \subseteq \SS$ so that the trajectory $\zeta$ will run into a terminal state in finite steps under any policy $\pi$. Formally, let $T \define \inf \{ t\geq1: S_t\in\terminals \}$ be the \textbf{termination time} (which is a random variable in the probability space $\P_{\pi}$), a finite-time task is an MDP with $\E[\pi][T]<\infty, \forall  \pi$. The goal is to find an \textbf{optimal policy} that maximizes the expected total reward until termination:

\vspace{-0.1in} \small
\eq[objective]{
J(\pi) \define \Exp[\zeta\sim\pi]{\sum\nolimits_{t=1}^{T(\zeta)}~ R(S_t)}
}
\normalsize \noindent
Finite-time tasks are important for both empirical and theoretical reasons. Empirically, they account for most real-world tasks in current AI practice. See Section \ref{sec:elp_mt} for a MDP formulation of \emph{machine translation} as example. One should not confuse finite-time tasks with \emph{finite-horizon} tasks. The latter is a special case of the former, in which $T$ has a finite upper bound $H$ (so $\E[\pi][T]\leq H$). Alternatively, a task in which every state has a non-zero probability $1-\gamma$ to reach terminal states immediately (under any policy) does not have a finite horizon, yet it is a finite-time task because in such task $T$ follows a geometric distribution with finite mean. Moreover, if in this case $R(s)=0$ for all terminal states, maximizing the total-reward objective \eqref{objective} with termination would be equivalent to maximizing the discounted reward $\E[\pi][\sum_{t=1}^{\infty} R(S_t) \cdot \gamma^{t-1}]$ without termination (see Appendix \ref{sec:more_about_mdp}). In this sense, the canonical discounted-MDP tasks can be recast into a special case of finite-time decision tasks with (undiscounted) total-reward objective \eqref{objective} too.

%Let $\Pi$ denote the \textbf{policy space}, i.e., the set of all policies. Without loss of generality, we assume every state $s$ in the state space $\SS$ is reachable from the initial state under at least one policy,~\footnote{Unreachable states are irrelevant to the real world whatsoever, so are excluded from our treatment.} where \emph{reachable} means $\exists \pi\in\Pi,~ \sum_{t=0}^\infty \Pr[\pi]{S_t=s} > 0$. A MDP is finite if both $\SS$ and $\AS$ are finite sets. 

Learning for finite-time decision task is typically based on trajectory data from \emph{multiple} decision episodes of the task, each starting from a state following $\rho$ and ending at a terminal state in $\terminals$. In imitation learning setting, the actions in the episodes may be provided by an expert policy; in this case the agent can learn from other's experiences. In reinforcement learning setting, the agent has to count on itself to generate the actions in its learning data. 

This \emph{learning process} -- in which the agent looks at experience of either others or itself, episode after episode -- can be formulated as a special family of finite-time MDPs where any terminal state transits back to a random state following the initial distribution $\rho$ under whatever action.\footnote{
	The ``transit back'' setting is fully aligned with the definition of finite-time MDP, and with the objective \eqref{objective} too, as neither of them prescribes what happens after the termination (except that the time homogeneity of MDP requires the transition to continue).
} 
A rollout trajectory of such recurrent MDP consists of an \emph{infinite} sequence of \emph{finite} episodes. Following \citet{2020:bojun}, we call such an ``MDP formalism of learning'', an \emph{Episodic Learning Process} (ELP). Formally, an MDP is an ELP if (1) it is a finite-time MDP (i.e. $\E[\pi][T]<\infty, \forall \pi$), (2) all terminal states ``reset'' the MDP in a homogeneous manner: $P(s'|s_1,a_1)=P(s'|s_2,a_2) = \rho(s'),~ \forall s_1,s_2\in\terminals,~ \forall a_1,a_2\in\AS,~ \forall s'\in \SS$, and (3) for mathematical convenience, the MDP is set to start at step $0$ from a terminal state $s_0\in \terminals$ (in that case $S_1$ still follows $\rho$). The ELP thus defined formally characterizes a learning problem for finite-time decision task, in which the learning algorithm obtains a (multi-episode) trajectory from the ELP, via either observing a given expert policy or trying an evolving behavior policy of its own, and seeks to find a good policy with respect to objective function \eqref{objective}, using the trajectory data.

%Besides its precise correspondence to real-world tasks and practice, the ELP framework also admits nice theoretical properties thanks to the inborn ELP conditions (1)-(3). For example, it is proved that every ELP is \emph{irreducible}, \emph{positive-recurrent}, and thus has a unique \emph{stationary distribution} $\rho_\pi$ under any behavior policy $\pi$~\cite{2020:bojun}. 
%See Proposition \ref{prop:irreducible}, \ref{prop:recurrent}, \ref{prop:change_space} in Appendix \ref{sec:proof} for the formal statements.

A Q-function assigns a real number to each state-action pair $(s,a)$ as the ``perceived benefit'' of doing $a$ under $s$. We use $\mathcal{Q}=\{\SS\times\AS\ra \Real\}$ to denote the set of all possible Q-functions. Each Q-function induces a greedy policy, denoted by $\pi_Q$, for which $\pi_Q(a|s) > 0$ only if $Q(s,a) = \max_{\bar{a}} Q(s,\bar{a})$. A Q-function is an \textbf{optimal Q-function} if the greedy policy $\pi_Q$ is an optimal policy. 

Bellman optimal value function is a special optimal Q-function that is characterized by the Bellman optimality operator. In its general form~\cite{2012:opac,2011:horde,2014:qlambda,2016:emphatic}, a generalized Bellman optimality operator $\Bopt[\ga]: \Q \ra \Q$ transforms a  function $Q$ into another function $\Bopt[\ga] Q$ such that for all $(s,a)\in \SS\times\AS$,

\vspace{-0.1in} \small
\eqm[bellman]{
\Bopt[\ga] Q (s,a) 
%&\define \E[S'\sim\T(s,a)] \Exp[A'\sim\pi(S')]{R(S') + \ga(S') \* Q(S',A')} \\
&\define \sum_{s'\in\SS}~ \T(s'|s,a) \*  \Big( R(s') + \ga(s') \* \max_{a'\in\AS}~ Q(s',a') \Big)
} 
\normalsize \noindent
where $\ga: \SS \ra [0,1]$ is a \textbf{discounting function} over the states. If $\gamma(s)<1$ on every state $s$, the corresponding Bellman optimality operator has a unique fixed point in the Q-space of any MDP, and this fixed point is known as the Bellman optimal value function. 

However, this fixed point under uniform discounting is generally not an optimal Q-function with respect to the undiscounted objective \eqref{objective}. 
In the following, we present a new theorem that generalizes the classic theory of Bellman optimal value to a more general class of discounting functions, among which a particular $\gamma$-function makes the Bellman optimal value precisely an optimal Q-function w.r.t. \eqref{objective}. 

\begin{theorem}
\label{thm:Qstar}
In any finite Episodic Learning Process $(\SS, \AS, \T, \R, \rho)$, let $\gamma$ be any discounting function such that $\gamma(s) < 1$ at all terminal state $s\in \terminals$, then 
\renewcommand\labelenumi{(\theenumi)}
\begin{enumerate}
\item $\Bopt[\ga]$ has a unique fixed point in $\Q$. 
%i.e. the equation $Q = \B^\gamma Q$ has a unique solution.

\item The fixed point of $\Bopt[\ga]$ is the limiting point of repeatedly applying $\Bopt[\ga]$ to any $Q\in\Q$.
%$Q^*=\lim\limits_{n \ra \infty} \underbrace{\Bopt[\ga] \dots \Bopt[\ga]}_{n} Q$.
%%$Q^*=\lim\limits_{n \ra \infty} \Big( \Bopt[\ga] \Big)^n Q$.

\item The fixed point of $\Bopt[\ga]$ is an optimal Q-function to \eqref{objective} when $\gamma$ is the following \textbf{episodic discounting function}:

\vspace{-0.2in} \small
\eq[gamma]{
\gammaEPI(s) ~\define~ \indicator{s\not\in\terminals} ~=~ 
\begin{cases}
	~~1~~ & \text{~for~} s \not\in \terminals \\
	~~0~~ & \text{~for~} s \in \terminals
\end{cases}
.}
\normalsize \noindent
\end{enumerate}
\end{theorem}
Different from the classic result which crucially relies on uniform discounting, both the uniqueness and optimality in Theorem \ref{thm:Qstar} are rooted from an inherent graph property of ELP. See the proof in Appendix \ref{sec:proof_bellman}.

Since we focus on optimizing the objective \eqref{objective} in this paper, we will write $\Bopt \define \Bopt[\gammaEPI]$, and simply call it the \textbf{Bellman operator}, for the particular Bellman optimality operator that uses the particular episodic discounting function \eqref{gamma}. We call the fixed point of $\Bopt$, the \textbf{Bellman value}, and denote it by $Q^*$. Similarly, we call $Q = \Bopt Q$, the \textbf{Bellman equation}, which is again assuming the particular episodic discounting function \eqref{gamma} in $\Bopt$. More explicitly, Bellman equation refers to the following non-linear equation over $\Q$:  

\vspace{-0.15in} \small
\eq[boe]{
Q(s,a) = \sum_{s'\in\SS}~ \max_{a'\in \AS}~ \T(s'|s,a) \* \Big( R(s') + \indicator{s'\not\in\terminals} \* Q(s',a') \Big)% \quad,~ \forall (s,a) \in \SS \times \AS
}
\normalsize \noindent
%The so-called \emph{value-based} episodic learning algorithms generally aim at solving the Bellman optimality equation \ref{boe} -- without knowing the transition function $\T$ or the reward function $\R$ -- through the episodic learning process, which gives the Bellman optimal value $Q^*$, from which the agent computes an optimal policy.
It is worth noting that although the Bellman value is unique, there can be more than one optimal Q-functions for a given problem. For example, any Q-function that gives the same preferential order with the Bellman value is also optimal. 
%Note that neither optimal policy function nor optimal Q-function needs to be unique. In general, an optimal Q-function may induce an infinite number of greedy policies (when it admits multiple optimal actions for some reachable states); and also, there can be an infinite number of optimal Q-functions that all give the same preference ordering in terms of decision making.  

%\begin{definition}
%$Q^*$ is the fixed point of $\Bopt$, with $Q^* = \Bopt Q^*$, where $\Bopt$ uses the episodic $\ga$-function \eqref{gamma}.
%\end{definition}

\section{A Nonlinear Lagrangian Duality}
\label{sec:lagrangian}

In this and the next section, we discuss a variational treatment to the Bellman equation which converts \eqref{boe} to a constrained \emph{minimization} problem. We will demonstrate some nice theoretical properties of this variational problem in this section, and will present algorithmic applications of the theory in the next section.

Our idea is inspired by a long-known linear programming re-formulation~\cite{1994:puterman} of the closely related Bellman equation for \emph{state-value functions} $V: \SS \ra \Real$,

\vspace{-0.1in} \small
\eq[bellman_v]{
V(s) = \max_a \sum_{s'} \T(s'|s,a) \* \Big( R(s') + \gamma \* V(s') \Big)% \quad,\quad \forall s
.}
\normalsize \noindent
Both the $Q$-form Bellman equation \eqref{boe} and the $V$-form Bellman equation \eqref{bellman_v} are nonlinear due to the $\max$ operator inside. But the $V$-form equation \eqref{bellman_v} admits a \emph{linear} variational re-formulation~\cite{2018:LP_RL}:%,2017:LP_DP, 2018:LP_RL

\vspace{-0.15in} \small
\eqm[lp_v]{
&\min_{V} \quad 	\sum_{s\in \SS} \rho(s) \* V(s)  \quad \\
&%\text{s.t.} \quad 	
V(s) \geq \sum_{s'\in\SS} \T(s'|s,a) ~ \Big( R(s') + \gamma ~ V(s') \Big)
%\text{subject to} \quad & Q(s,a) \geq \E[S'\sim\T(s,a)]~ \max_{a'\in \AS} \Big[ R(S') + \ga(S') \* Q(S',a') \Big]
%\quad,\quad \forall (s,a)
~,~ \forall (s,a)
}
\normalsize \noindent
Thanks to its linearity, \eqref{lp_v} enjoys the standard LP duality properties, and in particular has \emph{minimax equality} for its corresponding Lagrangian function, which is the basis for a recently revived thread of research on the LP approach to MDP and RL~\cite{2016:chen_wang, 2017:LP_DP, 2017:LP_RL, 2018:LP_RL, 2020:fenchel_duality}. 
In traditional settings, with the V-function solution of \eqref{bellman_v} or \eqref{lp_v}, one could construct an optimal Q-function by averaging the V-values over the transition probabilities, obtaining $Q(s,a)=\sum_{s'\in \SS} P(s'|s,a)~V(s')$. 

In modern learning settings, however, directly applying this V-function based approach encounters additional obstacles. First, the transition probabilities $P(s'|s,a)$ are often unknown in learning settings, and learning these probabilities is itself a substantial challenge. Second, computing greedy actions from the V-function often requires significantly more computational resource when deep neural networks are used. Specifically, we need to collect the V-values of successor states for each action, which requires to evaluate the V-function under at least $|\AS|$ different states even assuming deterministic transition (and the number can be much higher under stochastic transition). On the other hand, to decide the greedy action for given state $s$, we only need to evaluate the Q-function under one state, the state $s$. Further evaluating the Q-values for each action \emph{under the same state} usually involves much less computation.\footnote{
	For very large action spaces, such as continuous spaces, optimizing over even state-conditioned Q-values can be non-trivial~\cite{2016:naf}, but in that case optimizing over the V-values under different states would be only more forbiddingly expensive.
}

For these reasons, most value-based learning algorithms focus on learning the Q-functions directly~\cite{1989:qlearning,2010:double-q,2015:dqn,2018:sac}.
It is thus natural to ask if we can develop a variational treatment directly to the $Q$-form Bellman equation \eqref{boe}, similar to what we did to the $V$-form equation. 

Indeed, the $Q$-form Bellman equation \eqref{boe} can be similarly recast into a constrained optimization problem, such as in the following form:
\eqm[minQ]{
\min_{Q} \quad & 	\Exp[\zeta\sim \pi]{Q(S_T,A_T)}  \\
%\text{s.t.} \quad & 	Q(s,a) \geq  \sum_{s'}~ \max_{a'}~ \T(s'|s,a) \* \Big( R(s') + \gammaEPI(s') \* Q(s',a') \Big) 
\text{s.t.} \quad &	Q(s,a) \geq \Bopt Q(s,a)
%\text{subject to} \quad & Q(s,a) \geq \E[S'\sim\T(s,a)]~ \max_{a'\in \AS} \Big[ R(S') + \ga(S') \* Q(S',a') \Big]
\quad,\quad \forall (s,a)
}
where $\pi$ is an arbitrary policy, and $T$ is the termination time.
Unfortunately, unlike the $V$-form Bellman equation for which the nonlinear operation $\max_a$ can be ``unpacked'' into $|\AS|$ linear constraints (cf. \eqref{bellman_v} and \eqref{lp_v}), the optimization problem \eqref{minQ} for the $Q$-form Bellman equation is still nonlinear, due to the $\max_{a'}$ ``wrapped'' inside $\sum_{s'}$. As a result, although \eqref{minQ} can still be written into the Lagrangian form,
it is unclear if the nonlinear Lagrangian still enjoys strong duality, a key property for designing effective and principled learning algorithms~\cite{2017:LP_DP, 2018:LP_RL}.

In the following, we give an affirmative answer to this open question by proving a \emph{minimax theorem} for the nonlinear Lagrangian of the $Q$-form Bellman equation \eqref{boe}. %As with Theorem \ref{thm:Qstar}, the strong Lagrangian duality is also rooted from inherent structures of the episodic learning process. 

\begin{definition}
\label{def:lagrangian}
Given a finite ELP $(\SS,\AS,\T,\R,\rho)$, 
%a function $\Lagr_\pi: \mathcal{Q} \times \Real_{\geq 0}^{|\SS|\* |\AS|} \ra \Real$ is called 
the \textbf{Lagrangian function} with \textbf{conjugate policy} $\pi$ is $\Lagr_\pi(Q,\mul)\define$

\vspace{-0.1in} \small
\eq[lagrangian]{
\Exp[\zeta\sim \pi]{Q(S_T,A_T)} + \sum_{s,a} \lambda(s,a) ~ \Big( \Bopt Q(s,a) - Q(s,a) \Big)  
}
\normalsize \noindent
%where $\pi\in\Pi$ can be an arbitrary policy, and $T \define \inf \{ t\geq1: S_t\in\terminals \}$ is the termination time.
\end{definition}
Note that the conjugate policy $\pi$ of the Lagrangian only affects the first term of $\Lagr_\pi$. The second term of $\Lagr_\pi$ always uses the nonlinear \emph{optimality} operator $\Bopt$ which takes max over the Q-values, regardless of the conjugate policy $\pi$.

The first term of the Lagrangian takes average only over the terminal states, instead of over \emph{all} states as in dynamic programming~\cite{1994:puterman,2017:LP_DP}, because the latter is typically intractable in learning settings -- the learning agent who relies on the episodic rollout to obtain data cannot sweep through the whole state-action space itself, not even sample the space (near-)uniformly.\footnote{
	Mathematically, running a policy with uniform (or positive) distribution over the actions would have non-zero probability to reach any state; but typically in reality, such explorations will ``almost surely'' be stuck in a very limited region of the state space.
} 
On the other hand, the Lagrangian averages over the \emph{terminal} states, instead of over the \emph{initial} states as in previous learning-oriented works (e.g. ~\cite{2017:LP_RL,2018:LP_RL,2020:fenchel_duality}), because the distribution of initial states are fixed, while which terminal states a policy would reach depends on the behavior of the policy. It turns out that this policy dependency, as well as the mathematical construct of evaluating the Q-function over the terminal states (which is only meaningful in the ELP formulation), would play an important role in developing the main results we shall see.

Let us first confirm, with the following lemma, that the Bellman value function $Q^*$ is a minimax solution of the Lagrangian function \eqref{lagrangian}. See the proof in Appendix \ref{sec:proof_Qstar_minimax}. 

\addtocounter{theorem}{1}
\begin{lemma}
\label{lem:Qstar_minimax}
In any finite ELP, for any conjugate policy $\pi$, %the Bellman value function $Q^*$ is a minimax saddle solution of the Lagrangian $\Lagr_{\pi}$:
\eq{
	Q^* \in \arg~ \min_{Q\in\Q} \max_{\mul\geq 0}~ \Lagr_\pi(Q, \mul)
.}
\end{lemma}

Now we show that $\Lagr_\pi$ has a dual form when the Lagrangian multiplier 
%``conjugates with'' the (conjugate) policy $\pi$ 
is set to a special vector. Specifically, it is known that in any episodic learning process, every policy $\pi$ has a unique \textbf{stationary distribution} $\rho_\pi(s)$ such that $\Pr[\pi]{S_{t+1}=s} = \rho_\pi(s)$ if $\Pr[\pi]{S_{t}=s} = \rho_\pi(s)$~\cite{2020:bojun}. The following lemma gives a transformation of the Lagrangian when the multiplier vector $\mul$ is proportional to the stationary distribution of the conjugate policy $\pi$ (and is scaled by the average episode length of $\pi$). 

\begin{lemma}
\label{lem:lagr_dual}
In any finite ELP, let $\Lagr_\pi$ be the Lagrangian with conjugate policy $\pi$, and let $\mul_\pi$ be the particular Lagrangian multiplier with $\mul_\pi(s,a) = \rho_\pi(s) \* \pi(a|s) \* \E_{\pi}[T]$, then $\Lagr_\pi (Q, \mul_\pi) =$

\vspace{-0.1in} \small
\eqm[lagr_dual]{
%&	\Lagr_\pi (Q, \mul_\pi) = \\
&	J(\pi) + \sum_{s\not\in\terminals} \sum_{a\in\AS} 
	\mul_\pi(s,a) ~ \Big( \max_{\bar{a}} Q(s,\bar{a}) - Q(s,a) \Big)
}
\normalsize \noindent
\end{lemma}
\begin{proofidea}
Applying a known \emph{ELP ergodic theorem}~\cite{2020:bojun}, we can transform the first term of \eqref{lagrangian} from an average over trajectories to an average over the state-action space: 

\vspace{-0.1in} \small
\eq{
\Exp[\zeta\sim\pi]{Q(S_T,A_T)} 
= \E[\zeta\sim\pi] [T] \* \Exp[S,A\sim\rho_\pi]{\Big(1-\gammaEPI(s)\Big) ~ Q(S,A)}
%= \E_{\pi}[T] \* \sum_{s,a} \Big(1-\gammaEPI(s)\Big)~ Q(s,a)
}
\normalsize \noindent
Substituting the above into \eqref{lagrangian} and re-arranging would give \eqref{lagr_dual}. To derive the $J(\pi)$ term in \eqref{lagr_dual} (which is an average over the trajectory space, see \eqref{objective}), we will need to use the ELP ergodic theorem again to transform things back to the trajectory space. See Appendix \ref{sec:proof_lagr_dual} for the complete proof.
\end{proofidea}
\addtocounter{theorem}{-1}

The first term in the dual form of the Lagrangian, i.e. in \eqref{lagr_dual}, is the \emph{true} performance of the conjugate policy $\pi$, which is a constant with respect to the Q-function. Utilizing this fact, we can prove the strong duality property for Lagrangians with optimal conjugate policy.

\begin{theorem}%[ELP Minimax Theorem]
\label{thm:elp_minimax}
In any finite ELP $(\SS,\AS,\T,\R,\rho)$, if $\mu$ is an optimal policy, then its conjugate Lagrangian $\Lagr_\mu$ has strong duality property, for which
\eq[duality]{
\min_{Q\in\Q}~ \max_{\mul\geq 0}~ \Lagr_{\mu} (Q,\mul) = \max_{\mul\geq 0}~ \min_{Q\in\Q}~ \Lagr_{\mu} (Q,\mul) = J(\mu)
}
\end{theorem}
\begin{proofidea}
For any conjugate policy $\pi$, due to Lemma \ref{lem:Qstar_minimax}, $Q^*$ is a minimax solution of $\Lagr_\pi$, from which we can obtain $\min_{Q} \max_{\mul} \Lagr_\pi(Q,\mul) = \E[\pi] [Q_{\pi^*}(S_T,A_T)] = J(\pi_{Q^*})$, where $\pi_{Q^*}$ is the $Q^*$-greedy policy.

On the other hand, again for any conjugate policy $\pi$, due to Lemma \ref{lem:lagr_dual}, the dual form \eqref{lagr_dual} of the Lagrangian, if seen as a function of $Q$, attains its minimum when $Q$ achieves complementary slackness with $\mul_\pi$ in the second term, in which case the Lagrangian dual equals $J(\pi)$. Thus, for the particular multiplier $\mul_\pi$ we have $\min_{Q} \Lagr_\pi(Q, \mul_\pi) = J(\pi)$, which is a lower bound of $\max_{\mul} \min_{Q} \Lagr_\pi(Q, \mul_\pi)$.

Now, set the conjugate policy $\pi$ to be the optimal policy $\mu$, as assumed in the theorem, we would have 
\begin{align}
\max_{\mul} \min_{Q} \Lagr_\mu (Q,\mul) 
&\geq \min_{Q} \Lagr_\mu (Q,\mul_\mu) 
= J(\mu) = J(\pi_{Q^*}) \nonumber \\
&= \min_{Q} \max_{\mul} \Lagr_\mu(Q,\mul)  \label{elp_minimax_1}
.\end{align}
Because of the \emph{weak duality} of the Lagrangian (which universally holds for any function), we also have 
$\max_{\mul} \min_{Q} \Lagr_\mu (Q,\mul) \leq \min_{Q} \max_{\mul} \Lagr_\mu (Q,\mul)$, which means \eqref{elp_minimax_1} can only be an equality, thus gives \eqref{duality} as desired. See Appendix \ref{sec:proof_elp_minimax} for the complete proof. 
\end{proofidea}

Theorem \ref{thm:elp_minimax} is a minimax theorem for the Q-functions and the Lagrangian multipliers. But from the proof idea above, we can also see that $\pi_{Q^*}$, the greedy policy of the Bellman value, serves as the counterpart of $Q^*$ to form a minimax point of $\Lagr_\mu$. In fact, we can judge if any pair of Q-function and policy form a Lagrangian minimax point 
based on complementary slackness
%by combining the complementary slackness conditions in \eqref{lagrangian} and \eqref{lagr_dual}. 
(see the proof in Appendix \ref{sec:proof_eq_condition}):  
\begin{proposition}
\label{prop:eq_condition}
Given a finite ELP, for any Q-function $Q$ and any policy $\pi$, let $\rho_\pi(s,a) = \rho_\pi(s) ~ \pi(a|s)$ and $\mul_\pi(s,a) = \rho_\pi(s,a) ~ \E_{\pi}[T]$, we have 
\eq{
		\Lagr_{\pi} (Q, \bar{\mul}) 	
\leq 	\Lagr_{\pi} (Q, \mul_\pi)
\leq 	\Lagr_{\pi} (\bar{Q}, \mul_\pi)
\quad,\quad \forall \bar{Q}, \bar{\mul}
}
if and only if
\renewcommand\labelenumi{(\theenumi)}
\begin{enumerate}
\item $\Bopt Q(s,a) - Q(s,a) \leq 0$ \quad\quad\quad\quad\quad\quad\quad,~~ $\forall (s,a)$
\vspace{-0.05in}
\item $\rho_\pi(s,a) ~ \Big( \Bopt Q(s,a) - Q(s,a) \Big) = 0$ ~\quad\quad,~~ $\forall (s,a)$
\vspace{-0.05in}
\item $\rho_\pi(s,a) ~ \Big( \max_{\bar{a}} Q(s,\bar{a}) - Q(s,a) \Big) = 0$ ~,

	  \quad\quad\quad\quad\quad\quad\quad\quad\quad\quad\quad\quad\quad\quad\quad~~
	  $\forall s\not\in\terminals, \forall a$
\end{enumerate}
%\begin{enumerate}
%\item $\Bopt Q(s,a) - Q(s,a) \leq 0$ \hspace{0.95in}\quad,\quad $\forall (s,a)\in\SS\times\AS$
%\vspace{-0.05in}
%\item $\rho_\pi(s,a) \* \Big( \Bopt Q(s,a) - Q(s,a) \Big) = 0$ \;~~\quad\quad,\quad $\forall (s,a)\in\SS\times\AS$
%\vspace{-0.1in}
%\item $\rho_\pi(s,a) \* \Big( \max_{\bar{a}} Q(s,\bar{a}) - Q(s,a) \Big) = 0$ \quad,\quad $\forall s\not\in\terminals, a\in\AS$
%\end{enumerate}
\end{proposition}

\section{Lagrangian Minimization Algorithms}
\label{sec:mt}

From the last section we know that the $Q$-form Lagrangian $\Lagr_\mu$ has strong duality (where $\mu$ is optimal policy), and that the special Lagrangian multiplier $\mul_\mu$ constitutes a minimax saddle point of the Lagrangian, together with the solution of \eqref{minQ}. In other words, $\mul_\mu$ must have maximized the \emph{Lagrangian dual} function $\min_{Q} \Lagr_\mu(Q,\mul)$, that is,
\begin{align}
	\max_{\pi} J(\pi) 
&= 	\min_{Q} \max_{\mul\geq 0} \Lagr_\mu(Q,\mul) 
=	\max_{\mul\geq 0} \min_{Q} \Lagr_\mu(Q,\mul) \nonumber\\
&= 	\min_{Q} \Lagr_\mu(Q,\mul_\mu) \label{lagr_min}
.\end{align}
In light of \eqref{lagr_min}, a simple idea to solve the variational Bellman value problem \eqref{minQ} is to find the Q-function that minimizes the Lagrangian at $\mul_\mu$, that is, minimizes $\Lagr_\mu(Q,\mul_\mu)$. Such a Q-function would be, by \eqref{lagr_min}, the counterpart of $\mul_\mu$; the two together form a minimax point of the Lagrangian, and thus this Q-function would be a solution of \eqref{minQ}. 
 
Although the closed form of $\Lagr_\mu(Q,\mul_\mu)$ depends on an optimal policy $\mu$ (which might appear impossible to solve at a first glance, as coming up with such an optimal policy was our original goal of learning), notice that estimating the gradient of $\Lagr_\mu(Q,\mul_\mu)$ only requires sampling data from the optimal policy $\mu$, rather than explicit knowledge about how $\mu$ is constructed.
Specifically, for a parameterized model $Q(s,a;\w)$, we have 

\vspace{-0.2in} \small
\eqm[lagrad]{
&	\nabla_{\w} \Lagr_\mu(Q(\w),\mul_\mu) 
= 	\Exp[\zeta\sim\mu]{ \nabla_{\w} Q(S_T, A_T;\w)} + \\
&	\quad\quad\quad\quad~~~
\E[\zeta\sim\mu][T] \Exp[S,A\sim\rho_\mu]{~ \nabla_{\w} [\Bopt Q - Q] (S,A;\w) ~} 
}
\normalsize \noindent
where $[\Bopt Q-Q](s,a;\w) \define$

\vspace{-0.2in} \small
\eq{
	\Exp[S'\sim P(s,a)]{R(S') + \gammaEPI(S')~ \max_{a'} Q(S',a';\w)} - Q(s,a;\w)
}
\normalsize \noindent
in which the optimal policy $\mu$ only plays a role in determining the data distributions of the terminal state-actions $(S_T,A_T)$, of the termination time $T$, and of the transition variable $(S,A,S')$. 

This observation inspires a general imitation learning approach, named \emph{LAgrangian MINimization} (LAMIN) here, in which we try to collect some demonstration data from an optimal policy,\footnote{
	When only data from a sub-optimal policy is available, the policy (although sub-optimal with respect to the true task performance) \emph{is} optimal to an shaped imitation reward, so that our algorithm implicitly learns for this shaped reward, leading to policies as good as the given policy (in the best case).
} based on which we construct an estimator of the Lagrangian gradient \eqref{lagrad}, then apply standard stochastic gradient procedures to find a local minimum of $\Lagr_\mu(Q(\w), \mul_\mu)$. The obtained $Q(\w)$ is an approximation to the variational solution \eqref{minQ} of the Bellman equation.

\subsection{Practical algorithms}
\label{sec:lamin}

One may design different estimators of \eqref{lagrad} to power the LAMIN optimization. In the following we discuss two (families of) such gradient estimators, as demonstrating examples; they both perform reasonably and similarly well in our validation experiments presented in Section \ref{experiment}.

\textbf{LAMIN1}: A technical challenge of estimating \eqref{lagrad} is to deal with the $\max$ operation in the Bellman operator $\Bopt$. One popular trick is to ``soften'' the max operation with a Boltzmann distribution with small temperature $\beta$:

\vspace{-0.1in} \small
\eq[lamin1]{
[\Bopt Q-Q](s,a;\w) \approx 
\E[S'\sim P(s,a)]~ \Exp[A'\sim \pi^\beta_{\w}(S')]{\delta(s,a,S',A';\w)}
}
\normalsize \noindent
where $\pi^\beta_{\w}(a|s) \define \frac{\exp\big( Q(s, a;\w) / \beta \big)}{\sum_b \exp\big( Q(s, b;\w) / \beta \big)}$~, and $\delta$ is the temporal difference error with $\delta(s,a,s',a';\w) \define R(s') + \gammaEPI(s') Q(s',a';\w)-Q(s,a;\w)$. 

%\vspace{-0.1in} 
%\small
%\eqm[lamin1]{
%&	\Lagr_\mu^\beta(Q(\w),\mul_\mu) = \E_\mu [Q(S_T, A_T;\w)] + \E_\mu [T] \*\\
%&	\E[S,A,S'\sim\rho_\mu]~ \Exp[A'\sim \pi^\beta_{Q(\w)}(S')]{\delta(S,A,S',A';\w)}
%}
%\normalsize \noindent

Let $\Lagr_\mu^\beta$ denote the smoothed Lagrangian that uses the right-hand side of \eqref{lamin1}. $\Lagr_\mu^\beta$ is readily differentiable \emph{and} can be arbitrarily close to the exact Lagrangian $\Lagr_\mu$ with $\beta\ra 0$ (e.g. with $\beta=0.01$, $\pi_1/\pi_2 > 20,000$ if $Q_1-Q_2>0.1$). 

The LAMIN1 algorithm simply seeks to minimize the smoothed Lagrangian $\Lagr_\mu^\beta(Q(\w),\mul_\mu)$ with SGD, based on unbiased estimator of $\nabla_{\w} \Lagr_\mu^\beta(Q(\w),\mul_\mu)$ (which can be analytically computed). See pseudo-code in Appendix \ref{sec:lamin1}. 

As with all SGD algorithms with unbiased gradient estimator, LAMIN1 is guaranteed to converge to a local optimum of the smoothed Lagrangian for any differentiable Q-model (subject to standard assumptions~\cite{2016:DL}). Moreover, LAMIN1 can converge to globally \emph{optimal Q-function} in some ``simple'' cases, such as when tabular models are used (even though the smoothed Lagrangian function is \emph{not} convex in that case):
\begin{proposition}
\label{prop:lamin1}
LAMIN1 converges to an optimal Q-function with respect to objective \eqref{objective} if $Q(\w)$ is a tabular model with $Q(s,a;\w) = w_{s,a}$.
\end{proposition}
Proposition \ref{prop:lamin1} implies that the $\beta$-smoothing trick in LAMIN1 is more than a heuristic for approximate optimization, but may also serve as a practical \emph{correction} to an inherent bias of the minimax points of the Lagrangian that we will discuss in Section \ref{sec:maximin}. See Appendix \ref{sec:lamin1} for 
%the proof of Proposition \ref{prop:lamin1}, as well as some 
more nuanced discussion on LAMIN1 and its optimality property.

\textbf{LAMIN2}: Another common trick is to construct ``local'' gradient estimator for the (original) Lagrangian $\Lagr_{\mu}$, where the estimation is unbiased, in a per-step sense, for ``most'' of the optimization steps. 
Specifically, let $\w_t$ be the parameter vector that is to be updated at SGD step $t$, we can estimate the value of $\nabla_{\w} [\Bopt Q-Q](s,a;\w) ~\big|_{\w=\w_t}$, for this particular parameter $\w_t$, with 

%\vspace{-0.1in} \small
%\eqm[lamin2']{
%	\nabla_{\w} \Big( \E_\mu [Q(S_T, A_T;\w)] + \E_\mu [T] \\
%	\E[S,A,S'\sim\rho_\mu]~ \Exp[A'\sim \pi_{Q(\w^*)}(S')]{\delta(S,A,S',A';\w)} \Big) ~\Big|_{\w=\w_t}
%%=~&	\nabla_{\w}~ \E_\mu [Q(S_T, A_T;\w)] + \E_\mu [T] \*
%%	\Exp[S,A\sim\rho_\mu]{\B_{\pi_{Q(w^*)}} Q(S,A;\w)-Q(S,A;\w)} ~\Big|_{\w=\w^*} \\
%%=~&	\E_\mu [\nabla_{\w} Q(S_T, A_T;\w)] + \E_\mu [T]
%%	\Exp[S,A,S'\sim\rho_\mu \\ A'\sim \pi_{Q(\w^*)}(S')]{\gammaEPI(S') \nabla_{\w} Q(S',A';\w) - \nabla_{\w} Q(S,A;\w)} ~\Big|_{\w=\w^*} \\
%}
%\normalsize \noindent
\vspace{-0.2in} \small
\eqm[lamin2']{
&			\nabla_{\w}~ [\Bopt Q-Q](s,a;\w)  ~\big|_{\w=\w_t} \\
\approx&	\E[S'\sim P(s,a)]~ \Exp[A'\sim \pi_{Q(\w_t)}(S')]{\nabla_{\w}~\delta(s,a,S',A';\w)} ~\big|_{\w=\w_t}
}
\normalsize \noindent
where $\pi_{Q(\w_t)}$ is the $Q(\w_t)$-greedy policy. Note that the greedy policy $\pi_{Q(\w_t)}$ in \eqref{lamin2'} does not depend on $\w$ and thus is invariant to $\nabla_{\w}$; in contrast, the Boltzmann policy $\pi^\beta_{\w}$ in LAMIN1 (see \eqref{lamin1}) will be differentiated by $\nabla_{\w}$.

We remark that despite the approximation sign in \eqref{lamin2'}, the LAMIN2 estimator is expected to enjoy \emph{exact} consistency for ``most'' of the time in the SGD process. Specifically, %we have
\begin{proposition}
\label{prop_lagrad}
\eqref{lamin2'} becomes exact equality if for the given Q-function parameter $\w_t$, $Q(\w_t)$ suggests a unique best action $a_{\max}(s;\w_t) \define \arg\max_{a\in\AS}~ Q(s,a;\w_t)$ for state $s$.
\end{proposition}
See the proof in Appendix \ref{sec:lamin2}. In practice, the condition in Proposition \ref{prop_lagrad} -- i.e. that $Q(\w_t)$ gives unique best actions -- shall be the normal case for large-scale and continuously-valued models, as it should be rare that the real-numbered Q-values of two actions would happen to be \emph{identical} under a model parameter that is being stochastically optimized.\footnote{The same argument applies to all ReLU neural networks too.} %; moreover, even when the condition does not hold, that $Q(s,a_1;\w^*)$ happens to be exactly identical to $Q(s,a_2;\w^*)$ for some frequently occurring state $s$, the error would be bounded under small learning rate. 

Finally, the local gradient estimator \eqref{lamin2'} can be further combined with the $\beta$-smoothing trick, that is, replacing the greedy policy $\pi_{Q(\w_t)}$ in \eqref{lamin2'} with the Boltzmann policy $\pi^\beta_{Q(\w_t)}$. See the resulted pseudo-code in Appendix \ref{sec:lamin2}. Empirically, modestly higher temperatures help slightly improve the performance of LAMIN2 in our experiments.

\begin{figure}[t]
	\centering
   	\includegraphics[width=1.0\linewidth]{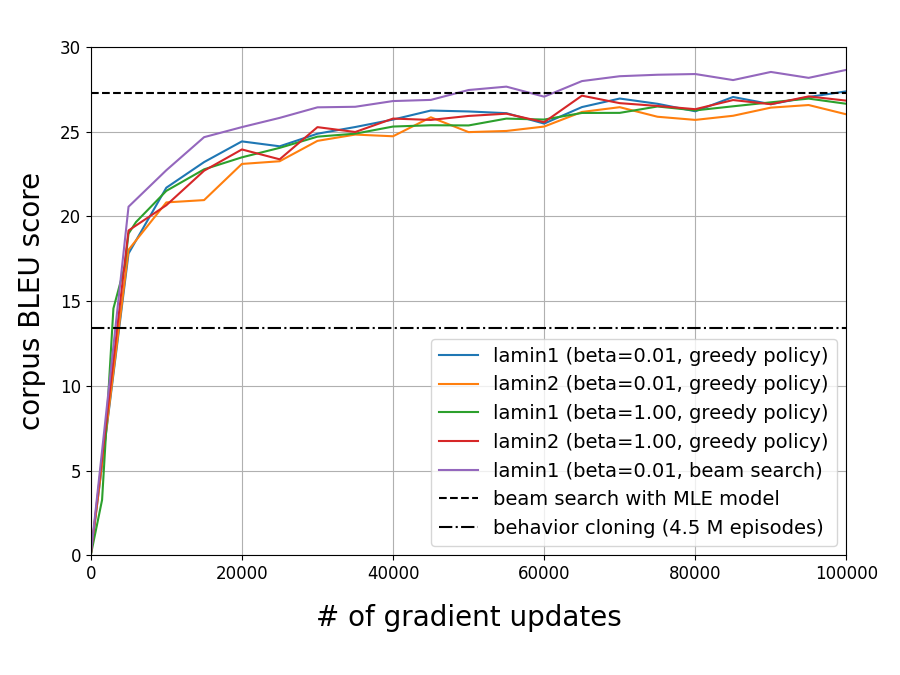}
    \caption{Learning curves of LAMIN. BLEU scores are measured on held-out test set. The dashed line indicates the standard baseline performance of $27.3$, achieved by \citet{2017:transformer} with a beam-search policy (beam size =$4$) using standard supervised learning. The purple line runs the same beam-search policy but with model trained by LAMIN1. Other lines use greedy policy, which is about 2x faster than search. The dash-dot line indicates the performance of behavior cloning, which only achieves $13.4$.}
	\label{fig:result}
\end{figure}

\subsection{Applications to Machine Translation}
\label{experiment}
As a case study, we now apply LAMIN algorithms to Machine Translation (MT), which is an application with significant real-world impact~\cite{2016:gnmt}, and is also an excellent example of episodic learning problem: A translation episode starts with a given sentence of a source language, and the agent takes actions to generate translation tokens one by one sequentially. The episode terminates when the agent outputs a special end-of-sentence (EOS) token, at which point the translation quality is evaluated. 
%Note that in MT, the length of translation sentence (i.e. the episode length) is a variable controlled by the agent policy, and standard MT metrics such as BLEU~\cite{2002:bleu} (i.e. the episode-wise return) tend to discount shorter (rather than longer) translations/episodes, for example through the \emph{brevity penalty factor}. 
See Appendix \ref{sec:elp_mt} for the complete ELP formulation of machine translation.

Given a source sentence $X$, most MT metrics measure the quality of a generated translation $Y$ by comparing the similarity between $Y$ and some \emph{reference translation} $Z$ of $X$, where $Z$ is provided by human expert. It follows that the policy used by the human expert (which maps $X$ to $Z$) must be an optimal policy under such metric. A trajectory of this optimal policy, in the form of a sequence of source-reference sentence pairs, is indeed widely available in standard MT benchmarks, which can be readily used to power LAMIN algorithms. In general, the same idea applies to all machine learning problems where the performance metric is a similarity evaluation against reference/groundtruth outputs.%; we can use standard supervised training data as an empirical sample of the distribution $\rho_\mu$ in gradient estimation of the Lagrangian function.

We tested our algorithms using the WMT 2014 English$\ra$German dataset, one of the most influential MT benchmarks. We parameterized value function $Q$ with the standard TransformerBase model~\cite{2017:transformer}, which contains about 65 million model parameters. The action space consists of 37000 \emph{unstructured} actions in each decision step (corresponding to 37000 tokens in the vocabulary of the target language). The training data consists of 4.5 million translation episodes made by human translators in European Parliament and multi-lingual news websites~\cite{2014:wmt}. 
%A translation episodes typically consists of 20-100 action steps. 
We trained the Transformer model using LAMIN1 and LAMIN2, with varying temperature $\beta$, then tested the Q-greedy policy with the standard BLEU metric. See Appendix \ref{sec:exp_setup} for more details on experiment setup.

\begin{figure}[t]
	\centering
	\includegraphics[width=1.0\linewidth]{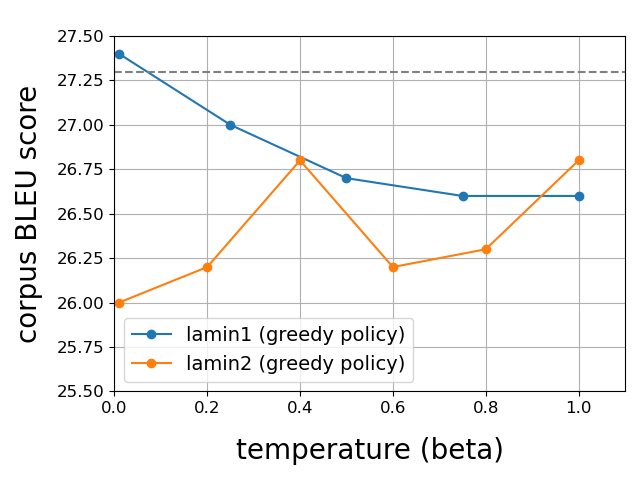}
	\caption{Performance of LAMIN1 and LAMIN2 under different temperature $\beta$. 
%Numerical scores are reported in Table \ref{tab:result_lamin1} and \ref{tab:result_lamin2} in Appendix \ref{sec:exp_setup}. 
Greedy policy is used in all cases. %The gray dashed line indicates $27.3$, the baseline performance for beam search (which is more time consuming).
}
	\label{fig:temperature}
\end{figure}

%\vspace{-0.1in}
%\begin{minipage}{\linewidth}
%  	\centering
%  	\begin{minipage}{0.45\linewidth}
%	%\vspace{40pt}
%	\begin{figure}[H]
%    	\includegraphics[width=\linewidth]{fig/laeq_vs_lamin}
%        \caption{Corpus BLEU scores of LAMIN1 and LAMIN2 under different Boltzmann temperature $\beta$. Greedy policy is used in both cases. The gray dashed line indicates the baseline performance of $27.3$ achieved by a beam-search policy in \cite{2017:transformer}.}
%		\label{fig:temperature}
%   	\end{figure}
%	\end{minipage}
%  	\hspace{0.04\linewidth}
%  	\begin{minipage}{0.45\linewidth}
%    \begin{figure}[H]
%    	\includegraphics[width=\linewidth]{fig/learning_curves}
%        \caption{Learning curves of some LAMIN variants. Beam search with a beam size of $4$ is used for the orange line, which matches the same setting with the baseline result of $27.3$ (gray dashed line). Other lines use greedy policy which is about 2x faster.}
%		\label{fig:result}
%   	\end{figure}
%  	\end{minipage}
%\end{minipage}
%\vspace{0.1in}

Figure \ref{fig:result} shows the learning curves of LAMIN1 and LAMIN2, under temperature $0.01$ (as a close approximation of the exact Lagrangian function) and $1.0$ (corresponding to the softmax distribution). We see that all the variants demonstrate similar learning curves. Notably, LAMIN1 with $\beta=0.01$, which is slightly better among the variants, attains a BLEU score of $27.4$ \emph{with greedy policy}. In comparison, with standard supervised learning, the same model famously attains $27.3$ on the same data set \emph{only if} further combined with a systematic beam search over the solution space~\cite{2017:transformer}. On the other hand, when also incorporating with the same beam search procedure, the Q-function trained by LAMIN1 (with $\beta=0.01$) attains $28.7$, which is $1.4$ BLEU (=$5\%$) higher. 

Notably, policy-based learning algorithms do not seem to work well in this task. As Figure \ref{fig:result} shows, behavior cloning, as a popular imitation learning baseline that mimics the expert policy via cross-entropy loss then samples the actions by the learned policy~\cite{2016:gail}, only achieves 13.4 BLEU (even being fed with 4.5 million demonstrations, a data size that is significantly larger than the typical ones used in many imitation learning research~\cite{2021:iq}). On the other hand, policy-based RL algorithms such as policy gradient are known to have difficulties in getting effective learning on this benchmark~\cite{2020:chosen}.

Finally, Figure \ref{fig:temperature} illustrates how the smoothing temperature $\beta$ affects the performance of LAMIN1 and LAMIN2. Table \ref{tab:result_lamin1} and \ref{tab:result_lamin2} in Appendix \ref{sec:exp_setup} reports the numerical scores. We see that LAMIN1 has slightly better performance than LAMIN2 under most temperatures. Interesting, higher temperature seems to help the performance of LAMIN2 but hurt that of LAMIN1 (albeit with limited margins in both cases).

\section{Minimax Points vs Maximin Points}
\label{sec:maximin}

\begin{figure}[t]
\centering
\includegraphics[width=0.6\linewidth]{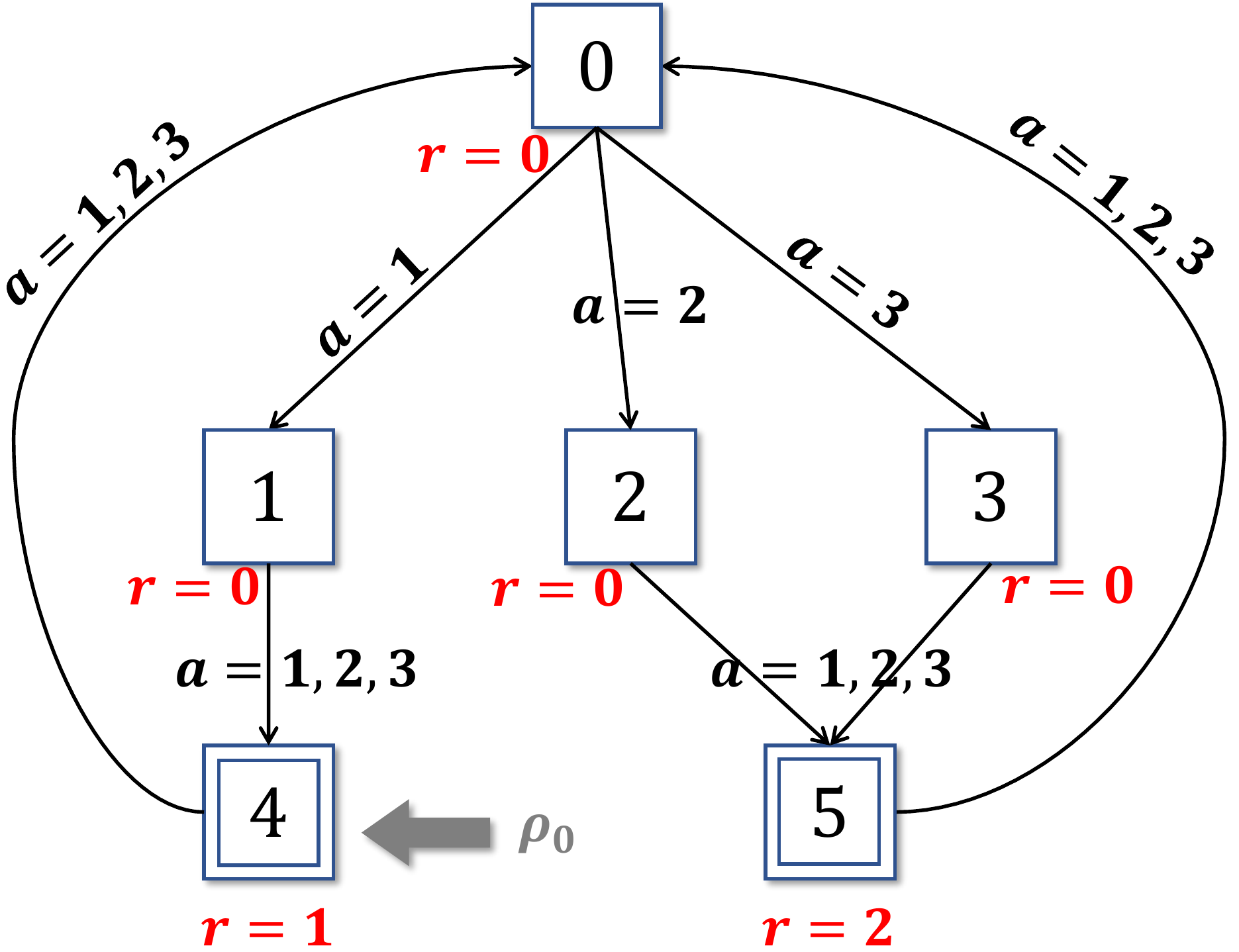}
\caption{An example for the suboptimality of minimax Q-functions. State $4$ and $5$ are terminal states, from which any action leads to state $0$. Choosing action $1,2,3$ under state $0$ deterministically transits to state $1,2,3$, respectively. 
%All actions under state $1$ lead to state $4$, and all actions under state $2$ and $3$ lead to state $5$. 
The agent only receives non-zero rewards at terminal states, with $R(4)=1$, $R(5)=2$. %The initial state is set to state $4$.
}
\label{fig:qexample}
\end{figure}

Our exploration of the Lagrangian method so far has been focusing on a particular class of Lagrangian saddle points, the minimax points, i.e. $\arg\min_Q\max_{\mul}\Lagr_\pi$. This follows a long tradition in the literature, dating back at least to \cite{1994:puterman} but perhaps to \cite{1960:LP_DP, 1963:LP_DP}, where the ``habit'' is to formulate the variational problem of Bellman equation into the minimization form, or the \emph{primal-form}~\cite{2017:LP_RL,2018:LP_RL,2020:fenchel_duality}, whose optimal solutions correspond to the minimax points of the Lagrangian. Let us call such Q-functions, \textbf{minimax Q-functions}. 

Indeed, the Bellman value $Q^*$, as an optimal Q-function, is a minimax Q-function (Lemma \ref{lem:Qstar_minimax}), and approximate solutions to find minimax Q-functions appear to empirically perform well, as our results in Section \ref{experiment} demonstrated. However, a minimax Q-function is not necessarily an optimal Q-function. 
Figure \ref{fig:qexample} gives an example. Clearly, an optimal policy in this ELP should only choose action $2$ or $3$ (or both), but not action $1$, under state $0$. On the other hand, one can verify that for this ELP, the constant Q-function $Q(s,a)\equiv 2$ is a minimax Q-function,  
%In particular, we have $2 = \Qmin(0,1) = 0 + 1 \* \Qmin(1,a) > 1+0\* \Qmin(4,a) = 1$, which does satisfy the constraints of \eqref{minQ} (that $Q \geq \Bopt Q$).
which is certainly not optimal as it assigns the same Q-value to all actions under state $0$. Moreover, the multiplier $\mul$ that counterparts with the constant Q-function (which together form a minimax point) needs to have $\mul(1,a)=0$ for all $a$. 
%(because $\Qmin(1,a) > \Bopt \Qmin(1,a) = 1$ for all $a$) 
Such a $\mul$ cannot encode a ``dual policy'' at all. See \ref{sec:example} for more details.  

In general, the problem with minimax Q-functions is that they only guarantee optimality for optimal actions, but do not enforce sub-optimality for sub-optimal actions (such as the action $1$ under state $0$ in Figure \ref{fig:qexample}). 
We remark that this is not limited to $Q$-functions or episodic/finite-time settings, but is a general issue rooted from the primal-form variational formulation. For example, the minimax solutions of $V$-form Lagrangian in Discounted-MDPs, as studied in previous works~\cite{2017:LP_RL,2018:LP_RL}, suffer from the same issue too (see Appendix \ref{sec:example} for a counter-example). 
This deficiency of minimax Q-functions might explain our empirical observation in the last section that adding a small temperature in the Bellman operator of the Lagrangian helps with the learning (because the Boltzmann averaging allows, to some extent, the optimizer to downgrade sub-optimal actions to further lower the Lagrangian).

Interestingly enough, it turns out that another class of Lagrangian saddle points -- the \emph{maximin points} -- can indeed guarantee the optimality of the corresponding Q-functions. Specifically, consider the following ``mirrored'' variational problem of the Bellman equation:
\eqm[maxQ]{
\max_{Q} \quad & 	\Exp[\zeta\sim \pi]{Q(S_T,A_T)}  \\
%\text{s.t.} \quad & 	Q(s,a) \geq  \sum_{s'}~ \max_{a'}~ \T(s'|s,a) \* \Big( R(s') + \gammaEPI(s') \* Q(s',a') \Big) 
\text{s.t.} \quad &	Q(s,a) \leq \Bopt Q(s,a)
%\text{subject to} \quad & Q(s,a) \geq \E[S'\sim\T(s,a)]~ \max_{a'\in \AS} \Big[ R(S') + \ga(S') \* Q(S',a') \Big]
\quad,\quad \forall (s,a)
}
Comparing with the primal-form variational problem \eqref{minQ}, which asks to find a ``minimal'' Q from a set of ``large'' Q's (in the sense that $Q\geq \Bopt Q$), the dual-form variational problem \eqref{maxQ} asks to find a ``maximal'' Q from a set of ``small'' Q's (in the sense that $Q\leq \Bopt Q$). It is easy to see that the Bellman value $Q^*$ is still an optimal solution of the dual-form problem, as the following lemma confirms (see the proof in Appendix \ref{sec:proof_Qstar_maximin}).

\begin{figure}[t]
\centering
\includegraphics[width=0.9\linewidth]{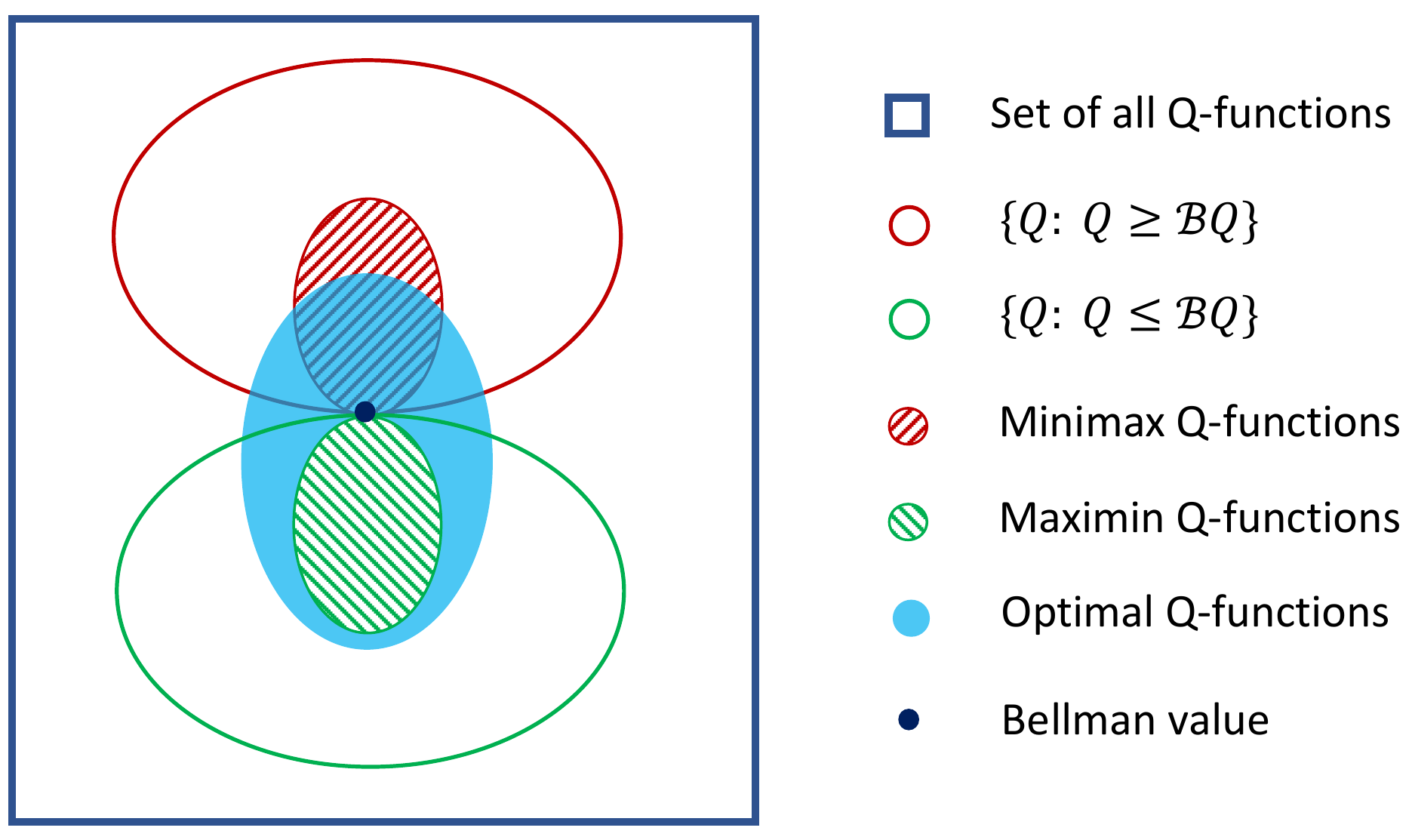}
\caption{A landscape of the Q-space, which combined conclusions from Lemma \ref{lem:Qstar_minimax}, Lemma \ref{lem:Qstar_maximin}, Figure \ref{fig:qexample}, and Theorem \ref{thm:Qopt}. 
%Note that an optimal Q-function may be neither minimax nor maximin; a trivial example is if the Q-function gives "wrong" Q-values at some terminal states that are unreachable by the greedy policy induced by the Q-functions (but that are reachable by the conjugate policy $\pi$).
}
\label{fig:qspace}
\end{figure}

\addtocounter{theorem}{1}
\begin{lemma}
\label{lem:Qstar_maximin}
In any finite ELP, for any conjugate policy $\pi$, $Q^*$ is an optimal solution of problem \eqref{maxQ}, or equivalently,
\eq{
	Q^* \in \arg~ \max_{Q\in\Q} \min_{\mul\geq 0}~ \Lagr_\pi(Q, \mul)
.}
%\eqm[maxQ]{
%\max_{Q} 	\quad  \Exp[\zeta\sim \pi]{Q(S_T,A_T)}  
%\quad \text{s.t.} \quad  Q(s,a) \leq  \Bopt Q(s,a) 
%~~,~~ \forall (s,a)
%}
%.
\end{lemma}
\addtocounter{theorem}{-1}

%Note that the conjugate policy $\pi$ does not change the minimax value of the Lagrangian, but will affect its maximin value. 

We call an optimal solution of \eqref{maxQ}, a \textbf{maximin Q-function}. Note that minimax and maximin Q-functions, as defined by \eqref{minQ} and \eqref{maxQ} (resp.), correspond to different types of saddle points of the \emph{same} function $\Lagr_\pi$. But different from minimax Q-functions, it turns out that a maximin Q-function always enforces sub-optimality for all (truly) sub-optimal actions \emph{and} at the same time can guarantee optimality for at least one (truly) optimal action. As a result, a maximin Q-function \emph{always and only} induces optimal policy.
\begin{theorem}
\label{thm:Qopt}
In any finite ELP, for any conjugate policy $\pi$, let $\Qmax$ be a maximin Q-function with respect to the Lagrangian $\Lagr_\pi$, then $\Qmax$ is an optimal Q-function.%, in the sense that $\Qmax$-greedy policy maximizes the episodic-reward objective \eqref{objective}.
\end{theorem}
\begin{proofidea}
Because $\Qmax$ is a feasible solution of \eqref{maxQ}, we have $\Qmax \leq \Bopt \Qmax \leq \Bopt\Bopt \Qmax \dots \leq Q^*$, so $\Qmax \leq Q^*$, which implies that a sub-optimal action for $Q^*$ (under a state) can only have even lower Q-values for $\Qmax$. Therefore, it is enough to prove that $\max_a \Qmax(s,a) = \max_a Q^*(s,a)$ at every non-terminal state $s$ \emph{that is reachable by a $\Qmax$-greedy policy} -- in that case under any non-terminal state that the $\Qmax$-greedy policy may encounter, $\Qmax$ is giving the same Q-values with $Q^*$ for optimal actions \emph{and} is giving even lower Q-values for sub-optimal actions, so the induced greedy policy must always choose the optimal actions. Note that $\Qmax$'s values on terminal states do not affect its footprint, due to the ELP conditions.

The value equality at (reachable) optimal actions can be proved by induction, starting from the terminal states for which the equality of $\Qmax$ and $Q^*$ is guaranteed by the objective function of \eqref{maxQ}. Then to enable a chain of induction, we will need to prove Proposition \ref{prop:Qmax_induction} in Appendix \ref{sec:proof_Qopt}, and use it as induction rule. See \ref{sec:proof_Qopt} for the complete proof. 
\end{proofidea}

Theorem \ref{thm:Qopt} reveals an interesting symmetry breaking in the Lagrangian method for Q-function learning: From the same Bellman equation, there are two symmetric ways to form the variational problems, both leading to the same Lagrangian function, yet only one gives optimal Q-functions, the other does not (necessarily). Figure \ref{fig:qspace} gives a diagrammatic summary of this phenomenon. The asymmetry appears to be rooted from the asymmetric Bellman operator $\Bopt$ which is only taking max, not min, over the Q-values. The maximin Q-functions have been largely overlooked in existing literature, but in light of their better optimality property as shown in this paper, we believe that the maximin Q-functions, as well as the corresponding \emph{dual-form} variational formulation \eqref{maxQ}, deserve more attentions from the community in future.

\section{Related works}
\label{sec:discussion}

A main challenge of this work comes from the non-linearity in both the Bellman optimality equation and the associated $Q$-form Lagrangian being studied. In contrast, most related research focused on linear treatments to related objects. 
For a linear \emph{policy-specific} Bellman operator $\B_\pi$ (which replaces $\max_a$ with a linear average), \citet{2017:unifying} proved that $\B_\pi$ has unique fixed point in episodic learning setting. The linear operator $\B_\pi$ leads to a LP re-formulation, based on which the ``DICE'' family of policy evaluation algorithms were developed~\cite{2019:dualDICE, 2020:bestDICE}. On the policy optimization side, an active thread of research used saddle-point optimization to solve the linear $V$-form Lagrangian, again relying on the generic LP duality inherited from the linear treatment~\cite{2016:chen_wang, 2017:LP_DP, 2017:LP_RL, 2018:LP_RL, 2018:chen_wang, 2019:LP_DP, 2003:LP_DP}. The underlying techniques in these linear settings are not directly applicable to the nonlinear problems studied in this work.

%\citet{2018:SBEED} considered a quadratically smoothed Lagrangian of the $V$-form optimality equation, and developed a primal-dual method to solve it. \citet{2010:REPS} proposed the REPS algorithm based on the Lagrangian formulation of a nonlinear constrained optimization problem. Algorithmic techniques developed in these works may potentially be used to compute the maximin saddle points of the Lagrangian function considered in this paper.

A popular Inverse Reinforcement Learning (IRL) framework of imitation learning also uses minimax saddle-point optimization to learn optimal policy and value functions~\cite{2016:gail}. \cite{2021:iq} developed a Q-function learning algorithm based on the IRL framework, which bears some similarity with the LAMIN algorithm developed in our paper. We however note that the rationale behind the two algorithms are completely different, and the Lagrangian method we discussed is not limited to the IRL problem.

The WMT machine translation benchmark used in this paper is a fruitful driver behind the rapid technical advances in Neural Machine Translation recent years~\cite{2016:gnmt,2017:six_challenges,2017:transformer}. MDP-based techniques have been actively studied as a promising method for this problem~\cite{2016:ranzato,2018:edunov,2017:bahdanau,2018:wu}, but with relatively limited empirical gain observed so far~\cite{2020:chosen}. To our best knowledge, the LAMIN algorithm is one of the first MDP-based solutions that is able to train Transformer-scale neural networks \emph{independently} (without the aid of other major techniques, such as supervised pretraining or ensemble learning) to attain competitive performance on the WMT benchmark. 

The disparity between the canonical discounted-reward formulation and common learning practice is a well recognized issue in reinforcement learning. The RL textbook \cite{2018:RL} devoted its Section 10.4 to the issue of deprecating the discounted formalism. The DP textbook \cite{1996:NDP} subsumed the discounted setting as a special case of an finite-termination setting. A special case of Theorem \ref{thm:Qstar} dedicated to episodic discounting was proved by \cite{1991:ssp}.

\section*{Acknowledgments}
The author of the paper would like to thank Fei Yuan and Longtu Zhang for inspiring discussions on this work, as well as for their help in preparing the WMT experiment.

%\small
\bibliographystyle{plainnat}
\bibliography{rlnmt,rl,elp}

\newpage
\appendix
\onecolumn

\newcommand{\tildeE}[1][]{\tilde{\mathbf{E}}_{#1}}
\newcommand{\tildeExp}[2][]{\mathbf{E}_{#1} \Big[ #2 \Big]}
\newcommand{\tildePr}[2][]{\mathbf{P}_{#1} [#2]}

\section{On the MDP Formulation}
\label{sec:more_about_mdp}
In our problem formulation, the rewards are conditioned on states, not explicitly on actions. However, the action dependency is implicitly captured by the action-conditioned state transition. In fact, given an MDP $M$ with action-conditioned state-reward transition function $P(s',r'|s,a)$ (from which all kinds of action-based rewards can be defined\cite{2018:RL}), we can construct an MDP $\tilde{M}$ with states $\tilde{s} \define (s,r)$, state(-only) transition function $\tilde{P}(\tilde{s}'|\tilde{s},a) = \tilde{P}\big( (s',r')|(s,r),a \big) \define P(s',r'|s,a)$, and state-based reward function $\tilde{R}(\tilde{s}) = \tilde{R}\big( s,r \big) \define r$. It is clear that $M$ and $\tilde{M}$ are encoding the same probability space of state-action-reward trajectories. The choice of state-based reward formulation is mostly for mathematical convenience. 

Similarly, although the objective function assumed in this paper does not explicitly encompass temporal discounting, the canonical exponential reward discounting -- if indeed required as a formulation of something real -- can be equivalently captured as a discounting implicitly from stochastic termination, as many works have pointed out (e.g. see \cite{2011:horde,2000:DPOC}). In the following we nevertheless provide a self-contained argument regarding the generality of our objective formulation.

Given a discounted-MDP $M$ that never terminates (i.e. $\terminals=\emptyset$) and that uses the discounted-reward objective $\E[\pi][\sum_{t=1}^{\infty} R(S_t) \cdot \gamma^{t-1}]$, we can construct an MDP $\tilde{M}$ by adding a zero-reward terminal state $s_\bot$, with $R(s_\bot)=0$, and modifying the transition function so that every state in $M$ transits to $s_\bot$ with probability $1-\gamma$ under any action. The modified MDP $\tilde{M}$ has finite average termination time $\E[\pi][\tilde{T}] = \sum_{l=2}^{\infty} l \cdot \gamma^{l-2}~(1-\gamma) = \frac{\gamma}{1-\gamma} + 2$, thus $\tilde{M}$ is a finite-time MDP. Moreover, the total-reward objective \eqref{objective} of $\tilde{M}$ equals the discounted-reward objective of $M$:
\eqm[total_to_discounted]{
%	J(\pi) 
 	\tildeExp[\pi]{\sum_{t=1}^T R(\tilde{S}_t)}
=& 	\sum_{l=1}^\infty \Big(~ 
		\tildeExp[\pi]{\sum_{t=1}^{l} R(\tilde{S}_t)|\tilde{T}=l} \cdot \P_\pi[\tilde{T}=l] 
	~\Big) \\
=& 	\sum_{l=2}^\infty \Big(~ 
		\underline{
			\tildeExp[\pi]{\sum_{t=1}^{l-1} R(\tilde{S}_t)|\tilde{T}=l}
		} \cdot \P_\pi[\tilde{T}=l] 
	~\Big) \\
=& 	\sum_{l=2}^\infty \Big(~ 
		\underline{
			\tildeExp[\pi]{\sum_{t=1}^{l-1} R(S_t)}
		} \cdot \P_\pi[\tilde{T}=l] 
	~\Big) \\
=& 	\sum_{l=2}^\infty \Big(~ 
		\tildeExp[\pi]{\sum_{t=1}^{l-1} R(S_t)} \cdot \gamma^{l-2}~(1-\gamma) 
	~\Big) \\
=& 	\tildeExp[\pi]{\sum_{m=1}^\infty\sum_{t=1}^m R(S_t) \cdot \gamma^{m-1}~(1-\gamma) } \\
=&	\tildeExp[\pi]{\sum_{t=1}^\infty \Big( \sum_{m=t}^\infty \gamma^{m-1}~(1-\gamma) \Big)\cdot R(S_t)} \\
=&	\tildeExp[\pi]{\sum_{t=1}^\infty \gamma^{t-1} \cdot R(S_t)}
}
Therefore, a policy that maximizes the total reward in the finite-time MDP $\tilde{M}$ will also maximize the discounted reward in the original MDP $M$. 

Note that in \eqref{total_to_discounted}, we have used the fact that $\Pr[\pi]{\tilde{T}=1}=0$ and that $R(s_\bot)=0$. In the key step highlighted by the underlines, we have $\E[\pi][\sum_{t=1}^{l} R(\tilde{S}_t)|\tilde{T}=l+1] = \E[\pi][\sum_{t=1}^{l} R(S_t)]$ because the condition $\tilde{T}=l+1$ only rules out early-terminating trajectories -- that is, rules out the possibility that some state $\tilde{S}_t$ for $t\leq l$ equals $s_\bot$ -- but the condition does not alter the probability ratios between trajectories that do not early terminate. More specifically, let $s_{1..l}$ denote an arbitrary trajectory segment $s_1\dots s_{l}$ that does not early terminate (i.e. $s_t\neq s_\bot$ for $1\leq t \leq l$), then
\eqa{
	\tildePr[\pi]{\tilde{S}_{1..l} = s_{1..l}|\tilde{T}=l+1} 
=~& \frac{
		\tildePr[\pi]{\tilde{T}=l+1|\tilde{S}_{1..l} = s_{1..l}}
	}{
		\tildePr[\pi]{\tilde{T}=l+1}
	}\cdot \tildePr[\pi]{\tilde{S}_{1..l} = s_{1..l}} \\
=~&	\frac{1-\gamma}{\gamma^{l-1}\cdot(1-\gamma)} \cdot \tildePr[\pi]{\tilde{S}_{1..l} = s_{1..l}} \\
=~&	\frac{1}{\gamma^{l-1}} \cdot \tildePr[\pi]{ \tilde{S}_t\neq s_\bot \text{~for~} t=1..l } \cdot 
	\tildePr[\pi]{ \tilde{S}_{1..l}=s_{1..l}| \tilde{S}_t\neq s_\bot \text{~for~} t=1..l } \\
=~&	\tildePr[\pi]{ \tilde{S}_{1..l}=s_{1..l}| \tilde{S}_t\neq s_\bot \text{~for~} t=1..l } \\
=~&	\Pr[\pi]{S_{1..l}=s_{1..l}}
}

\section{Proofs of the Bellman Optimality under Generalized Discounting}
\label{sec:proof_bellman}
In this section we prove Theorem \ref{thm:Qstar}, which asserts the uniqueness and optimality of the solution of generalized Bellman equation in episodic learning setting.

\subsection{Properties of Episodic Learning Process}
We start with presenting a series of mathematical properties of ELP; the first three are known, the rest are new. Most of our mathematical results in this paper are based on these properties.

\begin{proposition}
\label{prop:irreducible}
(\cite{2020:bojun}, Lemma 1.1)
For any policy $\pi$ in any ELP, the Markov chain induced by policy $\pi$ is irreducible: Let $s$ and $s'$ be any two states reachable by $\pi$, $\sum_{\tau=1}^\infty \Pr[\pi]{S_{t+\tau}=s'|S_t=s} > 0$. 
\end{proposition} 

\begin{proposition}
\label{prop:recurrent}
(\cite{2020:bojun}, Lemma 1.2)
For any policy $\pi$ in any ELP, the Markov chain induced by policy $\pi$ is positive recurrent: Let $s$ be any state reachable by $\pi$, $\E[\pi] [T_s] < \infty$, where $T_s$ is the recurrent time of $s$ in the markov chain of $\pi$. 
\end{proposition}

%Moreover, we will utilize the following proposition which allows to transform terms averaged over the trajectory space into those averaged over the stationary distribution.
\begin{proposition}
\label{prop:change_space}
(\cite{2020:bojun}, Theorem 4)
For any policy $\pi$ in any ELP, let $f: \SS \ra \Real$ be a real-valued function over the states, we have 
\eq[mle_laeq_5]{
\Exp[S\sim\rho_\pi]{f(S)} = \Exp[\zeta\sim\pi]{\sum_{t=1}^T f(S_t)} ~/~ \Exp[\zeta\sim\pi]{T}
}
\end{proposition}
%It is also proved in \cite{2020:bojun} that the policy gradient $\nabla J(\pi(\theta))$ can be written as an averaged term over the stationary distribution $\rho_{\pi(\theta)}$ 

For two Q-functions $Q_1,Q_2\in\Q $, we write $Q_1 \geq Q_2$ iff $Q_1(s,a)\geq Q_2(s,a)$, $\forall (s,a)\in\SS\times\AS$. The following proposition confirms that the generalized Bellman optimality operator \eqref{bellman} is \emph{monotonic}.

\begin{proposition}
\label{prop:monotonic}
For any discounting function $\ga:\SS \ra [0,1]$, the generalized Bellman optimality operator $\Bopt[\ga]$ is a monotonic operator over $\Q$, that is,  
$Q_1 \geq Q_2 \Rightarrow \Bopt[\ga] Q_1 \geq \Bopt[\ga] Q_2$~ for all $Q_1,Q_2\in\Q$.
\end{proposition}
\begin{proof}
Rewrite \eqref{bellman} as $\Bopt[\ga] Q(s,a) = \sum_{s'} \Big( \T(s'|s,a) \* \ga(s') \Big) \* \max_{a'} Q(s',a')  + \sum_{s'} \T(s'|s,a) R(s')$. As $Q_1(s',a') \geq Q_2(s',a')$ for all $(s',a')$, we have $\max_{a'} Q_1(s',a') \geq \max_{a'} Q_2(s',a')$ for all $s'$, and thus $\sum_{s'} \Big( \T(s'|s,a) \* \ga(s') \Big) \* \max_{a'} Q_1(s',a') \geq \sum_{s'} \Big( \T(s'|s,a) \* \ga(s') \Big) \* \max_{a'} Q_2(s',a')$, where $\T(s'|s,a)\* \ga(s') \geq 0$ for all $s'$.
\end{proof}

Now, as a well-known special case, if the discounting function is a constant $c$ less than $1$, the corresponding Bellman optimality operator $\B_c$ is a \emph{contraction mapping} with respect to the maximum-norm distance over $\Q$, which guarantees, by the Banach fixed-point theorem, that there is a special Q-function $Q^*_c$ which is both the unique fixed point and the unique limiting point of the Bellman optimality operator (i.e. $Q^*_c = \B_c Q^*_c$ and $Q^*_c = \lim\limits_{n \ra \infty} (\B_c)^n Q ~,~ \forall Q\in\mathcal{Q}$). 

Unfortunately, $\Bopt[\ga]$ loses the above contraction property under general discounting.
\begin{proposition}
\label{prop:contraction}
In any ELP where there is a single state $s^*$ and a single action $a^*$ such that doing $a^*$ under $s^*$ only goes to non-terminal states $s'$, i.e. $\sum_{s'\not\in\terminals}\T(s'|s^*,a^*) = 1$, the Bellman operator $\Bopt$ with episodic discounting \eqref{gamma} is \textbf{not} a contraction mapping with respect to the maximum-norm distance: For some $Q_1, Q_2\in\Q$,
$\max_{s,a} |Q_1(s,a)-Q_2(s,a)| = \max_{s,a} |\Bopt Q_1(s,a) - \Bopt Q(s,a)|$.
\end{proposition}
\begin{proof}
Consider two Q-functions with constant difference everywhere: $Q_1(s,a) \equiv Q_2(s,a) + \delta$, for some $\delta > 0$. Clearly, $\max_{s,a} |Q_1(s,a)-Q_2(s,a)| = \delta$. On the other hand, for any $(s,a)$ we have
$|\Bopt  Q_1(s,a) - \Bopt Q_2(s,a)| = \sum_{s'\in\SS}~ \T(s'|s,a) \* \ga(s') \* \delta \leq \delta$, which equals $\delta$ at the particular $(s^*,a^*)$ because $\sum_{s'\in\SS}~ \T(s'|s^*,a^*) \* \ga(s') \* \delta = \sum_{s'\not\in\terminals}~ \T(s'|s^*,a^*) \* 1 \* \delta = \delta$. %Thus $\max_{s,a} |\Bopt Q_1(s,a)- \Bopt Q_2(s,a)| = \delta$.
\end{proof}

Proposition \ref{prop:contraction} is a bad news for most non-trivial ELPs: It says that the generalized Bellman operators $\Bopt[\ga]$ -- and the episodic operator $\B$ in particular -- is not a contraction mapping unless we always has a chance to immediately terminate an episode no matter where we are (even when we are at terminal states, at which point the episode has not effectively started yet!).

However, it turns out that for a large class of the generalized Bellman operators (including the episodic operator $\B$), they still enjoy unique fixed and limiting point in \emph{all} finite ELPs, not because of the contraction property as in discounted-MDPs, but because of a graph property dedicated to the family of ELPs:

\begin{proposition}
\label{lem:closure}
Given an Episodic Learning Process $(\SS, \AS, \T, \R, \rho_0)$ and a policy $\pi$ in it, for any subset of states $\Omega\subseteq\SS$, let 
$\mathcal{C}_\pi(\Omega) \define \{~ s':~ \exists s\in\Omega,~ \Pr[\pi]{S_{t+1}=s'|S_t=s} > 0 ~\}$
be the set of all successor states that are one-step reachable from $\Omega$ under $\pi$, and let $\SS_\pi$
%$\SS_\pi \define \{~ s':~ \exists \zeta=(s_0,a_0,s_1,a_1,\dots,s_t,a_t,s'),~ \Pr[\pi]{\zeta} > 0 ~\}$
be the set of states that are ever reachable under $\pi$ (from initial states, in finite steps), then
$\mathcal{C}_\pi(\Omega) \subseteq \Omega$ only if $\SS_\pi \subseteq \Omega$.
\end{proposition}
\vspace{-0.1in}
\begin{proof}
For contradiction suppose $\mathcal{C}_\pi(\Omega) \subseteq \Omega$ and yet there is a $\pi$-reachable state $s^* \in \SS_\pi$ that is outside the given subset $\Omega$. We will show that in this case, it is possible to construct a policy $\mu$ (that is possibly different from $\pi$) such that $s^*$ is also reachable under $\mu$, and that $\mu$ admits an infinite trajectory that passes through $s^*$ and never return back to $s^*$ (see Figure \ref{fig:lem_closure_full} below). This would contradict with Proposition \ref{prop:recurrent} above which asserts that a $\mu$-reachable state $s^*$ must have finite mean recurrence time under $\mu$. 
\begin{figure}[h]
\centering
\includegraphics[width=0.4\linewidth]{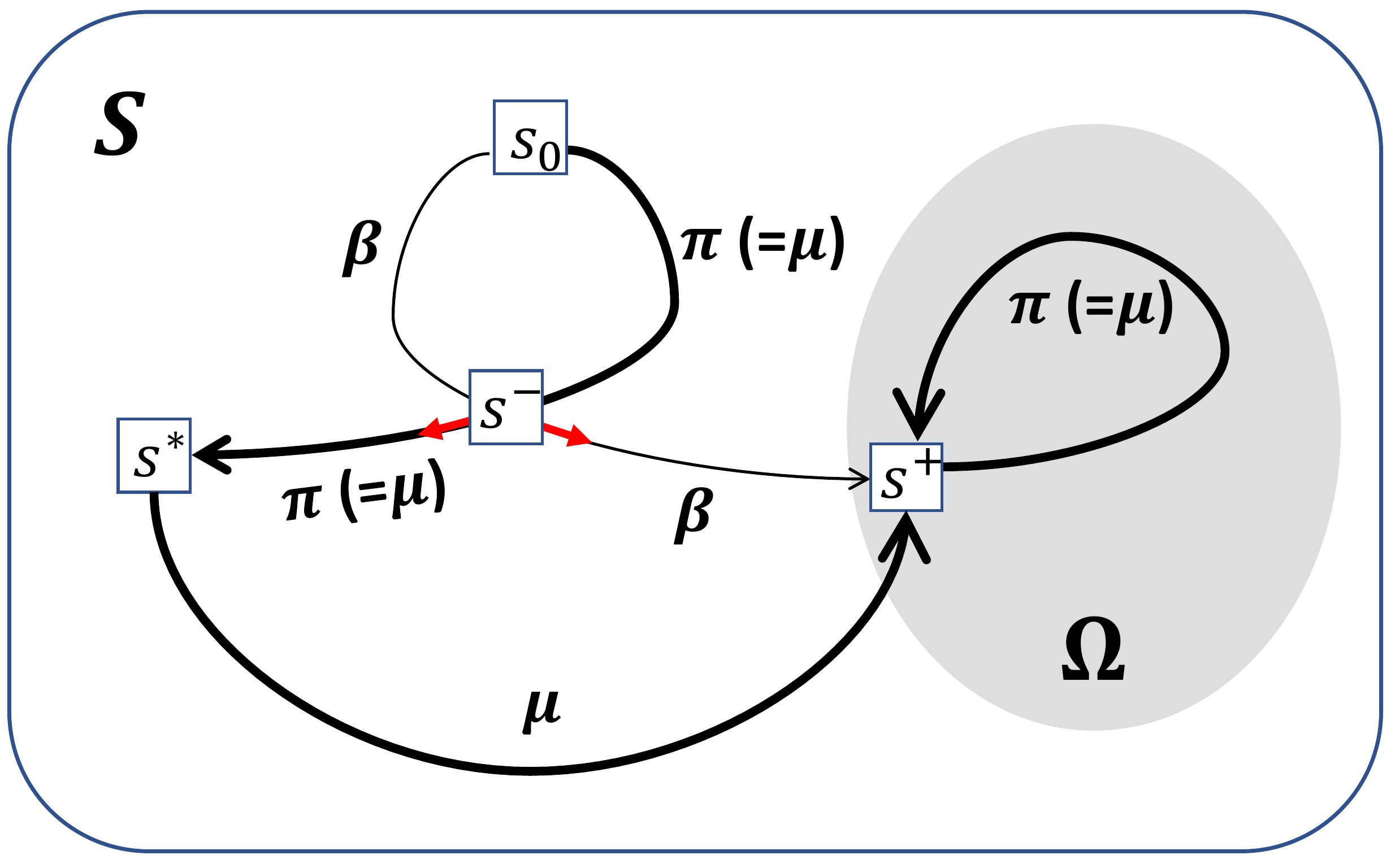}
\caption{An $\mu$-admissible trajectory that goes through $s^*$ but never returns.}
\label{fig:lem_closure_full}
\end{figure}

Specifically, first observe that $\mathcal{C}_\pi(\Omega) \subseteq \Omega$ means $\Omega$ is an absorbing subset under $\pi$ so that once we get into $\Omega$ we would never get out. 
In this case, the existence of a $\pi$-reachable state $s^*\not\in\Omega$ entails that $\Omega$ cannot contain all initial states, as otherwise from no initial state (in $\Omega$) we can go outside the absorbing subset $\Omega$ to reach $s^*$. Let $s_0$ be such an initial state that is outside $\Omega$, from which we can reach $s^*$ under $\pi$ (as assumed) without reaching any state in $\Omega$ in the middle (otherwise we never reach $s^*$). 

Now, pick an arbitrary state $s^+$ in $\Omega$, there must be some policy $\beta$ (not necessarily $\pi$) under which we can reach $s^+$ from $s_0$ (as states unreachable under \emph{any} policy should not be included in $\SS$ in the first place, see Section \ref{sec:preliminaries}). Without loss of generality we can again assume that we can reach $s^+$ from $s_0$ under $\beta$ without going through any other state in $\Omega$ in the middle because otherwise we can simply re-define $s^+$ to be the first state in $\Omega$ that we have encountered on the path (from $s_0$ to the ``old $s^+$'').

So far we have obtained an initial state $s_0$ outside $\Omega$, from which there is an admissible path $s_0 \overset{\pi}{\ra} s^*$ for policy $\pi$, and an admissible path $s_0 \overset{\beta}{\ra} s^+$ for policy $\beta$. Both paths only contain states outside $\Omega$ (except $s^+$). Now we construct policy $\mu$ as follows: we ask $\mu$ to copy $\pi$ on states in the path $s_0 \overset{\pi}{\ra} s^*$, and ask $\mu$ to copy $\beta$ on states in the path $s_0 \overset{\beta}{\ra} s^+$. If a state shows up in both paths -- such as the state $s^-$ in Figure \ref{fig:lem_closure_full} -- we ask $\mu$ to be a (probability) mixture of both $\pi$ and $\beta$ on that state. Clearly, we can reach both $s^*$ and $s^+$ from $s_0$ under $\mu$ (note that the policy mixing only decreases the probabilities to reach $s^*$ and $s^+$ but does not change their reachability). 

Since both $s^*$ and $s^+$ are in $\SS_\mu$, by Proposition \ref{prop:irreducible}, policy $\mu$ must admit a finite path from $s^*$ to $s^+$ too. Note that so far we have only prescribed $\mu$'s behavior \emph{outside} $\Omega$. Our final step is to ask $\mu$ to copy $\pi$ for all states in $\Omega$, so that $\Omega$ is also an absorbing subset for $\mu$, meaning that once we reach $s^+$ (from $s^*$), we will be stuck in $\Omega$ without going back to $s^*$ (which is outside $\Omega$ by our assumption at the beginning of the proof). In this way, we have constructed a policy $\mu$, under which we can first go from $s_0$ to $s^*$, then go from $s^*$ to $s^+$, and then be stuck in $\Omega$ forever without returning to $s^*$. A possibility of such an infinite trajectory under $\mu$ directly contradicts with Proposition \ref{prop:recurrent}.    
\end{proof}

\newpage
\subsection{Proof of Theorem \ref{thm:Qstar} (1) (2)}
Proposition \ref{lem:closure} says that in ELPs, a subset of $\SS$ can be an absorbing set under a policy only if it is ``big'' enough to have contained all reachable states of this policy. Consequently, there cannot exist an absorbing subset \emph{outside} the reachable set $\SS_\pi$. Utilizing this fact, we can prove the first two statements of Theorem \ref{thm:Qstar}.

\begin{recap}[Theorem \ref{thm:Qstar} (1) (2)]
In any ELP $(\SS, \AS, \T, \R, \rho)$ with finite state space $\SS$ and finite action space $\AS$, let $\gamma$ be any discounting function such that $\gamma(s) < 1$ for all terminal state $s\in \terminals$, then 
\renewcommand\labelenumi{(\theenumi)}
\begin{enumerate}
\item $\Bopt[\ga]$ has a unique fixed point, i.e., the equation $Q = \B^\gamma Q$ has a unique solution.

\item The fixed point of $\Bopt[\ga]$ is also the limiting point of repeatedly applying $\Bopt[\ga]$ to any $Q\in\Q$.
\end{enumerate}
\end{recap}
\begin{proof}
For any two Q-functions $Q_1$ and $Q_2$, consider their $L_\infty$-distance
\eq[]{
d(Q_1,Q_2) \define \max_{s\in\SS}~~ d_s(Q_1,Q_2)
}
where
\eq[]{
d_s(Q_1,Q_2) \define \max_{a\in\AS}~~ |Q_1(s,a) - Q_2(s,a)|
.}
As usual, we have
\begin{align}
d(\B^\ga Q_1, \B^\ga Q_2) 
&= 	\max_{s\in\SS}~ \max_{a\in\AS}~ \Big|~ 
		\sum_{s'\in\SS} P(s'|s,a) \* \ga(s') \* \Big( \max_{a'_1}~ Q_1(s',a'_1)- \max_{a'_2}~ Q_2(s',a'_2) \Big) 
	~\Big| \nonumber\\
&\leq 	\max_{s\in\SS}~ \max_{a\in\AS}~ \sum_{s'\in\SS} P(s'|s,a) \* \ga(s') \* \Big|~
		\max_{a'_1}~ Q_1(s',a'_1)- \max_{a'_2}~ Q_2(s',a'_2) 
	~\Big| \nonumber\\
&\leq 	\max_{s\in\SS}~ \max_{a\in\AS}~ \sum_{s'\in\SS} P(s'|s,a) \* \ga(s') \* \max_{a'}~ \Big|~
		Q_1(s',a')- Q_2(s',a') 
	~\Big| \nonumber\\
&= \max_{s\in\SS}~ \max_{a\in\AS}~ \sum_{s'\in\SS} P(s'|s,a) \* \ga(s') \* d_{s'}(Q_1,Q_2) \label{distance_bound}
.\end{align}
Traditionally, it was assumed that $\gamma(s') \equiv \ga_c < 1$ for all states, thus the $\gamma(s')$ term in \eqref{distance_bound} can be readily moved out of the sum, immediately yielding $d(\B Q_1,\B Q_2) \leq \ga_c \* d(Q_1,Q_2)$.
When $\gamma(s')$ is not constant and is allowed to be $1$ for non-terminal states, applying the operator $\B^\ga$ to $Q_1$ and $Q_2$ cannot guarantee to reduce $d(Q_1,Q_2)$, as discussed in Proposition \ref{prop:contraction}. 

However, by utilizing the graph property as proved in Proposition \ref{lem:closure}, we can show that $\B^\ga$ guarantees to reduce the per-state distance $d_s(Q_1,Q_2)$ at some ``support dimension'' $s$, so that if we repeatedly apply $\B^\ga$, the set of ``support states'' will become smaller and smaller and eventually become empty at which point the overall $L_\infty$-distance gets reduced (by the composite operator ``repeatedly applying $\B^\ga$''). %thus will still cause a contraction as long as we repeatedly apply $\B_\pi$ for enough times.

Specifically, for given $Q_1, Q_2 \in \mathcal{Q}$, we will identify a sequence of \emph{proper} subsets of states
\eq[support_set_1]{
\mathcal{S} = \texttt{d-support}(0) \supset \texttt{d-support}(1) \supset \texttt{d-support}(2) \supset \dots \supset \texttt{d-support}(|\SS|)
} 
such that
\eq[support_set_2]{
s \not\in \texttt{d-support}(k) ~\Rightarrow~ \forall i\geq k,~ 
d_s \Big( (\mathcal{B}^\gamma)^i Q_1 , (\mathcal{B}^\gamma)^i Q_2 \Big) < d(Q_1,Q_2)
}
for all $k\geq 0$. 
The construction of the subsets is by induction, and is based on the following insight:

\addtocounter{theorem}{-10} 
\begin{lemma}
\label{prop:progress}
Under the condition of Theorem \ref{thm:Qstar}~, if \eqref{support_set_2} holds for $k-1$, then
\eq{
\exists {s^*} \in \texttt{d-support}(k),~ \text{~such that~~} \forall i\geq k,~
d_{s^*} \Big( (\mathcal{B}^\gamma)^i Q_1 , (\mathcal{B}^\gamma)^i Q_2 \Big) < d(Q_1,Q_2)
}
\end{lemma} \addtocounter{theorem}{10}
\begin{proof}
We first refactor \eqref{distance_bound} a little bit, which actually holds in a per-state sense, so
\eqm[ds_bound]{
d_s(\B^\ga Q_1, \B^\ga Q_2) 
&\leq \max_{a\in\AS}~ \sum_{s'\in\SS} P(s'|s,a) \* \ga(s') \* d_{s'}(Q_1,Q_2) \\%\nonumber\\%\label{ds_bound_1}\\
&\leq \max_{s'\in\SS}~ d_{s'}(Q_1,Q_2) ~=~ d(Q_1,Q_2)
.}
Recursively applying \eqref{ds_bound} gives
\eqm[ds_bound_1]{
d_s \Big( (\mathcal{B}^\gamma)^k Q_1 , (\mathcal{B}^\gamma)^k Q_2 \Big) 
&\leq 
\max_{a\in\AS}~ \sum_{s'\in\SS} P(s'|s,a) \* \ga(s') \*
d_{s'} \Big( (\mathcal{B}^\gamma)^{k-1} Q_1 , (\mathcal{B}^\gamma)^{k-1} Q_2 \Big) 
\\ &=
\sum_{s'\in\SS} P(s'|s,a_{\max}(s)) \* \ga(s') \*
d_{s'} \Big( (\mathcal{B}^\gamma)^{k-1} Q_1 , (\mathcal{B}^\gamma)^{k-1} Q_2 \Big)
\\ &\leq
\sum_{s'\in\SS} P(s'|s,a_{\max}(s)) \* \ga(s') \* d(Q_1,Q_2) \leq d(Q_1,Q_2) 
}
where $a_{\max}(s) \define \arg\max_a ~ \sum_{s'\in\SS} P(s'|s,a) \* \ga(s') \*
d_{s'} \Big( (\mathcal{B}^\gamma)^{k-1} Q_1 , (\mathcal{B}^\gamma)^{k-1} Q_2 \Big)$.

Now we prove that the inequality \eqref{ds_bound_1} must be strict for some support-state $s^* \in \texttt{d-support}(k)$. In particular, we will prove
\eq[ds_bound_2]{
\sum_{s'\in\SS} P(s'|s^*,a_{\max}(s^*)) \* \ga(s') \*
d_{s'} \Big( (\mathcal{B}^\gamma)^{k-1} Q_1 , (\mathcal{B}^\gamma)^{k-1} Q_2 \Big) < d(Q_1,Q_2)
}

\textbf{Case $1$}: There is an $s^*\in \texttt{d-support}(k)$ with $\sum_{s' \in \texttt{d-support}(k)} \T \big( s'|s^*,a_{\max}(s^*) \big) < 1$.
In this case it's possible to go from $s^*$ to some $s'$ outside the subset of  $\texttt{d-support}(k)$. For such $s'$ we have 
$\gamma(s') \* d_{s'} \Big( (\mathcal{B}^\gamma)^{k-1} Q_1 , (\mathcal{B}^\gamma)^{k-1} Q_2 \Big) < d(Q_1,Q_2)$
(because \eqref{support_set_2} holds for $k-1$ as assumed), thus
\eqm{
&
\sum_{s'\in\SS} P(s'|s^*,a_{\max}) \* \ga(s') \*
d_{s'} \Big( (\mathcal{B}^\gamma)^{k-1} Q_1 , (\mathcal{B}^\gamma)^{k-1} Q_2 \Big) 
\\ <~ &
\sum_{s' \not\in \texttt{d-support}(k)} P(s'|s^*,a_{\max}) \* d(Q_1,Q_2) ~+~
\\ & 
\sum_{s' \in \texttt{d-support}(k)} P(s'|s^*,a_{\max}) \* \gamma(s') \* 
d_{s'} \Big( (\mathcal{B}^\gamma)^{k-1} Q_1 , (\mathcal{B}^\gamma)^{k-1} Q_2 \Big) 
\\ \leq~ &
\sum_{s' \not\in \texttt{d-support}(k)} P(s'|s^*,a_{\max}) \* d(Q_1,Q_2) + \sum_{s' \in \texttt{d-support}(k)} P(s'|s^*,a_{\max}) \* d(Q_1,Q_2)
\\ =~ &
d(Q_1,Q_2)
.} 

\textbf{Case $2$}: For all $s\in \texttt{d-support}(Q_1,Q_2)$,~ $\sum_{s' \in \texttt{d-support}(Q_1,Q_2)} \T \big( s'|s,a_{\max}(s) \big) = 1$. This is equivalent to say that there exists a policy -- which would choose $a_{\max}(s)$ under the corresponding $s$ -- such that it is \emph{impossible} to move from any state in $\texttt{d-support}(k)$ to a state outside $\texttt{d-support}(k)$ under this policy. In other words, let $\mu$ be such a policy, we have 
\eq[]{
\mathcal{C}_\mu \Big( \texttt{d-support}(k) \Big) \subseteq \texttt{d-support}(k)
} 
which, by Proposition \ref{lem:closure}, entails 
\eq[s_mu_in_d_support]{
\SS_\mu \subseteq \texttt{d-support}(k)
,}
\eqref{s_mu_in_d_support} literally says that the set of support-states $\texttt{d-support}(k)$, as an absorbing subset under $\mu$ as assumed in case 2, must contain all reachable states under $\mu$. By the definition of ELP, these $\mu$-reachable states must in turn contain at least one terminal state (otherwise we would not have finite episode under $\mu$ at all). Let $s_\perp \in \SS_\mu \subseteq \texttt{d-support}(k)$ be such a reachable terminal state under $\mu$. Since $s_\perp$ is reachable under $\mu$ (and $\mu$ chooses $a_{\max}(s)$ under each $s$), there must also be an $s^*\in\SS_\mu$ such that $\T(s_\perp|s^*,a_{\max}(s^*)) > 0$. Because $\gamma(s_\perp) < 1$ as assumed as the general condition of Theorem \ref{thm:Qstar}, we have 
%\eqm[gap_of_weights]{
\eqm{
&
\sum_{s'\in\SS} P(s'|s^*,a_{\max}) \* \ga(s') \*
d_{s'} \Big( (\mathcal{B}^\gamma)^{k-1} Q_1 , (\mathcal{B}^\gamma)^{k-1} Q_2 \Big) 
\\ \leq~ &
\sum_{s'\in\SS} P(s'|s^*,a_{\max}) \* \ga(s') \* d(Q_1,Q_2)
\\ =~&
\Big(~ P(s_\perp|s^*,a_{\max}) \* \ga(s_\perp) + 
\sum_{s'\in \SS \setminus \{s_\perp\}} P(s'|s^*,a_{\max}) \* \ga(s') ~\Big) \* d(Q_1,Q_2)
\\ <~&
\Big(~ P(s_\perp|s^*,a_{\max}) + 
\sum_{s'\in \SS \setminus \{s_\perp\}} P(s'|s^*,a_{\max}) \* \ga(s') ~\Big) \* d(Q_1,Q_2)
\\ \leq~&
\Big(~ P(s_\perp|s^*,a_{\max}) + 
\sum_{s'\in \SS \setminus \{s_\perp\}} P(s'|s^*,a_{\max}) ~\Big) \* d(Q_1,Q_2)
\\ =~&
d(Q_1,Q_2)
.} 

Now we have proved that $\exists s^* \in \texttt{d-support}(k)$,~ 
$d_{s^*} \Big( (\mathcal{B}^\gamma)^k Q_1 , (\mathcal{B}^\gamma)^k Q_2 \Big) < d(Q_1,Q_2)$. It is straightforward to verify that the same proof idea applies to all $i>k$ too (in case 1, we still have $\gamma(s') \* d_{s'} \Big( (\mathcal{B}^\gamma)^{i-1} Q_1 , (\mathcal{B}^\gamma)^{i-1} Q_2 \Big) < d(Q_1,Q_2)$
because \eqref{support_set_2} holds for $k-1$; in case 2, the existence of $s_\perp$ is a graph property that is independent of how many times $\B^\ga$ is applied).
\end{proof}

With Lemma \ref{prop:progress}, we can construct each subset $\texttt{d-support}(k)$ in \eqref{support_set_1} by removing the $s^*$ from $\texttt{d-support}(k-1)$. It's clear that \eqref{support_set_2} will hold for the sequence of support subsets thus constructed. In particular, note that Lemma \ref{prop:progress} holds for $k=0$ without the inductive condition (that \eqref{support_set_2} holds for $k-1$) because in this case it's impossible to go outside $\texttt{d-support}(0) = \SS$ as in Case 1, so only Case 2 is possible (and in this case the proof does not need the inductive condition).

Now, \eqref{support_set_1} implies that $\texttt{d-support}(|\SS|)$ is empty set. Substituting this observation into \eqref{support_set_2}, yields 
$d_s \Big( (\mathcal{B}^\gamma)^{|\SS|} Q_1 , (\mathcal{B}^\gamma)^{|\SS|} Q_2 \Big) < d(Q_1,Q_2)$ for all $s\in\SS$, which means 
$d \Big( (\mathcal{B}^\gamma)^{|\SS|} Q_1 , (\mathcal{B}^\gamma)^{|\SS|} Q_2 \Big) < d(Q_1,Q_2)$. In other words, the existence of the d-support sequence satisfying \eqref{support_set_1} and \eqref{support_set_2} means that the composite operator of ``repeatedly applying $B^\ga$ for $|\SS|$ times'' (i.e. $(\B^\ga)^{|\SS|}$) guarantees to reduce the $L_\infty$-distance of every Q-function pair in $\Q$.

Moreover, note that the initial values of $Q_1$ and $Q_2$ do not determine how much percentage the distance between them will reduce -- the values of $Q_1$ and $Q_2$ only affect what $a_{\max}$ is in \eqref{ds_bound_2}, There are $|\mathcal{S}| \cdot |\mathcal{A}| \cdot |\mathcal{S}|$ possible transition probabilities in total, and $|\mathcal{S}|$ possible $\gamma$-values in \eqref{ds_bound_2}, so no matter how $a_{\max}$ (which is a policy) and $\arg\max \mathbf{d}$ change over iterations, they just select a different subset from the $|\mathcal{S}|^3 \cdot |\mathcal{A}|$ possible terms. There are a (possibly very large yet) finite number of such subsets, thus when the sum of the subset is less than $1$, there must be an absolute upper bound $1-r_{\min}$ for the subset sum. This upper bound ratio of distance reduction may be extremely close to $1$, especially after repeating the process for $|\SS|$ times, but still, it is a definite upper bound smaller than $1$, which is enough to make the composite operator $(\B^\ga)^{|\SS|}$ a (possibly very weak) contraction mapping, and thus admits a unique fixed point.% (by the Banach fixed point theorem).

Our last step is to confirm that the unique fixed point of the composite operator $(\B^\ga)^{|\SS|}$ is also the \emph{unique} fixed (and limiting) point of the original Bellman operator $\B^\ga$. Clearly $\B^\ga$ cannot have two fixed points, as otherwise their distance could not get reduced by repeatedly applying $\B^\ga$ for $|\SS|$ times, so the key is to show that $\B^\ga$ does have \emph{a} fixed point. In fact, we will prove the following slightly stronger result: 
\addtocounter{theorem}{-10} 
\begin{lemma}
\label{prop_bootstrap}
Under the condition of Theorem \ref{thm:Qstar}, let $Q^*$ be the unique fixed point of $(\B^\ga)^{|\SS|}$,
\eq[bellman_limit]{
Q^* = \lim\limits_{n \ra \infty} (\B^\ga)^n Q ~,~~ \forall Q\in\mathcal{Q} 
}
and 
\eq[bellman_fp]{
Q^* = \B^\ga Q^*
}
\end{lemma} \addtocounter{theorem}{10}
\begin{proof}
For brevity we use $\B$ to denote the generalized Bellman operator in this proof. % although the result clearly generalizes to all $\B^\ga$)

We first prove that \eqref{bellman_limit} holds for a special family of Q-functions, that is, for all Q-functions $Q^-$ with $Q^- \leq \B Q^-$.~\footnote{Such $Q^-$ guarantees to exist. In fact, every \emph{on-policy value function} is such a $Q^-$; see \ref{sec:proof_optimality} for details.} Specifically, because $\B$ is monotonic operator, for any such $Q^-$ we have 
\eq[q_sequence]{
Q^- \leq \B Q^- \leq (\B)^2 Q^- \dots \leq 
(\B)^{|S|} Q^- \dots \leq 
(\B)^{2\cdot |S|} Q^- \dots \leq 
Q^*
}
where $Q^*$ is an supremum (but not necessarily the limit) of the sequence above because of the contraction property of $\B^{|\SS|}$.
%In other words, the whole sequence of $\{ (\B)^i Q^- \}$ is a monotonic sequence.
Recall that $Q_1\leq Q_2$ means $\forall s,a~~ Q_1(s,a)\leq Q_2(s,a)$ 
%and that $d(Q_1,Q_2) = \max_{s,a} |Q_1(s,a)-Q_2(s,a)|$
, so \eqref{q_sequence} implies that the Q-functions in it must have non-increasing distances to $Q^*$.
%that is, $d(Q,Q^*) \geq d(\B Q,Q^*) \geq d((\B)^2 Q,Q^*) \dots \geq d((\B)^{|S|} Q, Q^*) \dots \geq d((\B)^{2\cdot |S|} Q,Q^*) \dots \geq 0$. 
On the other hand, as the subsequence $Q^-,~ (\B)^{|\SS|} Q^-,~ (\B)^{2\* |\SS|} Q^-,~ (\B)^{3\* |\SS|} Q^-, \dots$ converges to $Q^*$, we know that for any $\epsilon > 0$, there exists an $i^*$ such that $d\Big( (\B)^{i^* \* |\SS|} Q^-, Q^* \Big) < \epsilon$, and thus $d \Big( (\B)^{i} Q^-, Q^* \Big) < \epsilon$ for all integer $i > i^*$ due to the monotonicity, which literally means that the overall sequence \eqref{q_sequence} also converges to $Q^*$.

Now we have obtained a Q-function (i.e. $Q^-$) starting from which (and only from which, for now) repeatedly applying $\Bopt$ (rather than $(\B)^{|\SS|}$) will converge to $Q^*$. In other words, we have $Q^*=\lim\limits_{n \ra \infty} (\B)^n Q^-$. As a result, we must also have $\B Q^*= \B \lim\limits_{n \ra \infty} (\B)^n Q^- = \lim\limits_{n \ra \infty} (\B)^n Q^- = Q^*$, thus we have proved $\B Q^* = Q^*$ (i.e. $Q^*$ is \emph{a} fixed point of $\B$).

The final step of the proof for Lemma \ref{prop_bootstrap} is to show that $\Bopt$ has a limiting point regardless of the initial Q-function (previously we have only proved this for the special initial Q-function $Q^-$): For any $Q\in \mathcal{Q}$, because $Q^*=\B Q^*$, we have $d(\B Q, Q^*) = d(\B Q, \B Q^*) \geq d(Q, Q^*)$, where the inequality is by \eqref{distance_bound}. Again, $d(\B Q, Q^*) \geq d(Q, Q^*)$ means the  sequence \eqref{q_sequence} has non-increasing distances to $Q^*$, but this time for all $Q\in\Q$. So, by the same logic as before, $Q^*$ must be the limit of the sequence $Q, \B Q, (\B)^2 Q, \dots$, for all $Q\in \mathcal{Q}$.
\end{proof}

The specific form of the $\ga$-function does not play a role in the proof of Lemma \ref{prop_bootstrap} as presented above, so the conclusion of the lemma -- i.e. \eqref{bellman_limit} and \eqref{bellman_fp} -- apply to all $\B^\ga$ that comply with the condition of Theorem \ref{thm:Qstar}, which marks the completion of the proof for statement (1) and (2) in Theorem \ref{thm:Qstar}.
\end{proof}

In comparison, the classic Bellman optimality property requires $\gamma(s)<1$ at \emph{every} state $s\in \SS$, while Theorem \ref{thm:Qstar} only requires $\gamma(s)<1$ at terminal states. On the other hand, the classic result applies to all MDPs, while Theorem \ref{thm:Qstar} is fundamentally based on unique structures of episodic learning process. Importantly, the episodic discounting function that we use in ELP -- i.e. \eqref{gamma} -- does satisfy the condition of Theorem \ref{thm:Qstar}.

\subsection{Proof of Theorem \ref{thm:Qstar} (3)} 
\label{sec:proof_optimality}
Now we prove statement (3) of Theorem \ref{thm:Qstar} which asserts that $Q^*$, as the unique fixed point of the Bellman operator under episodic discounting, is indeed an optimal Q-function: 

\begin{recap}[Theorem \ref{thm:Qstar} (3)]
In any finite ELP, let $Q^*$ be the fixed point of $\Bopt$ (i.e. the solution of \eqref{boe}), let $\pi^*$ be a policy such that $\pi^*(a|s) > 0$ only if $Q^*(s,a) = \max_{\bar{a}} Q^*(s,\bar{a})$, then $J(\pi^*) = \max_{\pi} J(\pi)$, where $J$ is the episodic-reward objective \eqref{objective}.
\end{recap}

\begin{proof}
Every policy $\pi$ is coupled with a (unique) \textbf{on-policy value function} $Q_\pi$ which is a special Q-function that assigns Q-values according to the conditional expectations of the episode-wise total reward under the policy $\pi$: %a Q-function $Q_\pi \in \Q$ is the on-policy value function of policy $\pi$ if
\eq[on_policy_value]{
Q_\pi(s,a) \define \Exp[\{S_t,A_t\}\sim \pi]{ \sum_{t=1}^T R(S_t) \Big| S_0=s,A_0=a } \quad,\quad \forall (s,a)\in \SS\times\AS	
}
Comparing \eqref{on_policy_value} with the definition of the episodic-reward objective $J$ (i.e. with \eqref{objective}), and with the ELP conditions, one can find that for any policy $\pi$, its performance score equals its on-policy value \emph{at terminal states} (and all terminal states must have the same on-policy value):
\eq[terminal_value]{
J(\pi) = Q_\pi(s_\perp,a) \quad,\quad \forall s_\perp\in\terminals ~,~ \forall a\in\AS
.} 
In the following we will prove that for the $Q^*$-greedy policy $\pi^*$, we have $Q_{\pi^*} \geq Q_\pi$ for all $\pi$, which by \eqref{terminal_value} entails that $J(\pi^*) \geq J(\pi)$.

First observe that $Q_{\pi^*} = Q^*$, that is, the Q-function $Q^*$ is the on-policy value function of the greedy policy induced by itself. Specifically, with the episodic $\ga$-function \eqref{gamma}, for any $(s,a)$, by the definition of $\pi^*$ we have
\begin{align}
Q^*(s,a) 
&= \B Q^*(s,a) 	\nonumber\\
&= \E[S'\sim P(s,a)]~ \max_{a'}~~ \Big[ R(S') + \ga(S') \* Q^*(S',a') \Big] 	\nonumber\\
&= \E[S'\sim P(s,a)]~ \Exp[A'\sim \pi^*(S')]{R(S') + \ga(S') \* Q^*(S',A')}		\label{on_policy_value_1}\\
&= \Exp[\{S_1,A_1\}\sim \pi^*]{ R(S_1) + \ga(S_1) \* Q^*(S_1,A_1) \Big| S_0=s,A_0=a}		\nonumber\\
&= \Exp[\{S_1,A_1,S_2,A_2\}\sim \pi^*]{ R(S_1) +\ga(S_1) R(S_2) + \ga(S_1)\ga(S_2) \* Q^*(S_2,A_2) ~\Big|~ S_0=s,A_0=a} \nonumber\\
&= \Exp[\{S_t,A_t\}\sim \pi^*]{ \sum_{t=1}^T R(S_t) \Big| S_0=s,A_0=a }	\nonumber\\
%&= \Exp[\{S_t,A_t\}\sim \pi^*]{ \sum_{t=1}^\infty R(S_t)\* \prod_{\tau=1}^{t-1} \ga(S_\tau) \Big| S_0=s,A_0=a } \nonumber\\
&= Q_{\pi^*} (s,a) \label{on_policy_value_2}
\end{align}

In above, the equations from \eqref{on_policy_value_1} to \eqref{on_policy_value_2} apply to not only $\pi^*$ but to any policy $\pi$ too. This means 
\begin{align*}
Q_\pi(s,a) 
&=~~ 	\E[S'\sim P(s,a)]~ \Exp[A'\sim \pi(S')]{R(S') + \ga(S') \* Q_\pi(S',A')} \\
&\leq~~ \E[S'\sim P(s,a)]~ \max_{a'}~~ \Big[ R(S') + \ga(S') \* Q_\pi(S',a') \Big] \\
&= \Bopt Q_\pi(s,a)	
\end{align*}
So, for any policy $\pi$, we have $Q_\pi \leq \Bopt Q_\pi$. 
In other words, $\{Q : \exists \pi, Q = Q_\pi\}$, the set of all on-policy value functions, is a subset of the set $\{Q : Q \leq \Bopt Q\}$.

Now we have known that $Q^*$ is a maximum Q-function of the set $\{Q \leq \Bopt Q\}$, and that this set contains the set of on-policy value functions as a subset. Because $Q^*$ itself is an on-policy value function (as $Q^* = Q_{\pi^*}$), it follows that $Q^*$ must also be a maximum Q-function of the subset $\{Q_\pi\}$. Thus we have proved that $Q_{\pi^*} = Q^* \geq Q_\pi$ for all $\pi$, as desired. 
\end{proof}

%Lastly, we remark that the same proof idea of Theorem \ref{thm:Qstar} also applies to prove the existence of unique fixed point (and limiting point) for the related \emph{policy-specific Bellman operator} (with episodic $\ga$-function) in all finite ELPs.

\section{Proofs of the Nonlinear Lagrangian Duality (Section \ref{sec:lagrangian})}
\label{sec:proof}

\subsection{Proof of Lemma \ref{lem:Qstar_minimax}}
\label{sec:proof_Qstar_minimax}

\begin{recap}[Lemma \ref{lem:Qstar_minimax}]
In any finite ELP, for any conjugate policy $\pi$,~ $Q^* \in \arg \displaystyle{\min_{Q\in\Q}~ \max_{\mul\geq 0}~} \Lagr_\pi(Q, \mul)$.
\end{recap}

\begin{proof}
Because $\B$ is monotonic (Proposition \ref{prop:monotonic}), for any $Q$ with $Q\geq \Bopt Q$ we have 
$Q\geq \Bopt Q \geq (\Bopt)^2 Q \dots \geq Q^*$, thus $Q^* \leq Q$ for all $Q\in\Q$. On the other hand, the objective function of the variational problem \eqref{minQ} is a probabilistic average over the Q-values at terminal states/actions, which must attain its minimum at $Q^*$ because $Q^*$ is per-state-action minimal. In other words, $Q^*$ is an optimal solution of the variational problem \eqref{minQ}.

By standard Lagrangian duality theory, a Q-function is an optimal solution of \eqref{minQ} if and only if it is a minimax solution of the Lagrangian \eqref{lagrangian}. This is because only Q-functions with $Q\geq\Bopt Q$ can prevent $\max_{\mul\geq 0} \Lagr_\pi(Q,\mul)$ from being tuned arbitrarily large (by $\mul$), and for those $Q$'s that satisfy the constraint (i.e. the complementary slackness condition), the second term in the Lagrangian would equal zero, rendering $\max_{\mul\geq 0} \Lagr_\pi(Q,\mul) = \Exp[\zeta\sim\pi]{Q(S_T,A_T)}$, which attains its minimum at $Q^*$ as just proved.
\end{proof}

\subsection{Proof of Lemma \ref{lem:lagr_dual}}
\label{sec:proof_lagr_dual}

\begin{recap}[Lemma \ref{lem:lagr_dual}]
In any finite ELP, let $\Lagr_\pi$ be the Lagrangian with conjugate policy $\pi$, and let $\mul_\pi$ be the particular Lagrangian multiplier with $\mul_\pi(s,a) = \rho_\pi(s) \* \pi(a|s) \* \E_{\pi}[T]$, where $\rho_\pi$ is the stationary distribution of $\pi$, then 
\eq{
	\Lagr_\pi (Q, \mul_\pi) = 
	J(\pi) + \sum_{s\not\in\terminals} \sum_{a\in\AS} 
	\mul_\pi(s,a) ~ \Big( \max_{\bar{a}} Q(s,\bar{a}) - Q(s,a) \Big)
}
\end{recap}

\begin{proof}
Define $f(s) \define \Exp[A\sim\pi(s)]{\indicator{s\in\terminals} \* Q(s,A)}$, by Proposition \ref{prop:change_space} we have
\eq{
\Exp[\zeta\sim \pi]{Q(S_T,A_T)} = \Exp[S_{1..T}\sim\pi]{\sum_{t=1}^T f(S_t)} = \E_{\zeta\sim\pi}[T] \* \E[S\sim\rho_\pi] \Exp[A\sim\pi(S)]{\indicator{S\in\terminals} \* Q(S,A)}
.}
So, for any $Q\in \mathcal{Q}$, we have
\eqm{
&	\Lagr_\pi(Q,\mul_\pi) \\
=&	\E_{\zeta\sim\pi}[T] \* \Exp[S,A\sim\rho_\pi]{\indicator{S\in\terminals} \* Q(S,A)} + 
	\E_{\zeta\sim\pi}[T] \* \Exp[S,A\sim\rho_\pi]{\B Q(S,A) - Q(S,A)} \\
=&  \E_{\zeta\sim\pi}[T] \* \Exp[S,A\sim\rho_\pi]{\indicator{S\in\terminals} \* Q(S,A) - Q(S,A)} + \\
&	\E_{\zeta\sim\pi}[T] \* \Exp[S,A\sim\rho_\pi]{\Exp[S'\sim\T(S,A)]{R(S') + \ga(S')\*\max_{a'}Q(S',a')}} \\
=&	- \E_{\zeta\sim\pi}[T] \* \Exp[S,A\sim\rho_\pi]{\indicator{S\not\in\terminals} \* Q(S,A)} + 
	\underline{ \E_{\zeta\sim\pi}[T] \* \E[S'\sim\rho_\pi] \Big[ R(S') } + \ga(S') \* \max_{a'} Q(S',a') \Big] \\
=& 	\underline{ \E_{\zeta\sim\pi}[T] \* \Exp[S\sim\rho_\pi]{R(S)} } + 
	\E_{\zeta\sim\pi}[T] \* \Exp[S,A\sim\rho_\pi]{
		\indicator{S\not\in\terminals} \* \Big( \max_a Q(S,a) - Q(S,A) \Big)
	} \\
=& 	\underline{ J(\pi) } + \E_{\zeta\sim\pi}[T] \* \sum_{s\in\SS,a\in\AS} 
	\rho_\pi(s) \* \pi(a|s) \* \indicator{s\not\in\terminals} \* \Big( \max_{\bar{a}} Q(s,\bar{a}) - Q(s,a) \Big) \\
=& 	J(\pi) + \E_{\zeta\sim\pi}[T] \* \sum_{s\in\SS \setminus \terminals}~ \sum_{a\in\AS}~ \rho_\pi(s) \*
	\pi(a|s) \* \Big( \max_{\bar{a}} Q(s,\bar{a}) - Q(s,a) \Big)
}
In above, $\E_{\zeta\sim\pi}[T] \* \Exp[S\sim\rho_\pi]{R(S)} = J(\pi)$ is obtained by applying the transformation of Proposition \ref{prop:change_space} again, this time with $f(s) \define R(s)$.
\end{proof}

\subsection{Proof of Theorem \ref{thm:elp_minimax}}
\label{sec:proof_elp_minimax}

\begin{recap}[ELP Minimax Theorem]
In any finite ELP $(\SS,\AS,\T,\R,\rho)$, if $\mu$ is an optimal policy, then its conjugate Lagrangian $\Lagr_\mu$ has strong duality property, for which
\eq{
\min_{Q\in\Q}~ \max_{\mul\geq 0}~ \Lagr_{\mu} (Q,\mul) = \max_{\mul\geq 0}~ \min_{Q\in\Q}~ \Lagr_{\mu} (Q,\mul) = J(\mu)
}
\end{recap}

\begin{proof}
Let $\pi^*$ be a $Q^*$-greedy policy, which is thus an optimal policy. %For convenience we use $J^*$ to denote the performance of optimal policies, thus $J^* = J(\pi^*) = J(\mu)$.
For any conjugate policy $\pi$, since $Q^*$ is a minimax solution of $\Lagr_\pi$ (Lemma \ref{lem:Qstar_minimax}), we have 
\eq{
	\min_{Q} \max_{\mul\geq 0} \Lagr_\pi(Q,\mul) 
= 	\max_{\mul\geq 0} \Lagr_\pi(Q^*, \mul) 
= 	\E[\zeta\sim\pi] [Q^*(S_T,A_T)]
.} 
By \eqref{on_policy_value_2} in \ref{sec:proof_optimality}, we have $\E[\pi] [Q^*(S_T,A_T)] = \E[\pi] [Q_{\pi^*}(S_T,A_T)]$. By \eqref{terminal_value} in \ref{sec:proof_optimality}, we further have $\E[\pi] [Q_{\pi^*}(S_T,A_T)] = J(\pi^*)$, even for $\pi\neq \pi^*$. Connecting these equations together, gives 
\eq[minimax_L]{
\min_{Q} \max_{\mul\geq 0}~ \Lagr_\pi(Q,\mul) = J(\pi^*) \quad,~~ \forall \pi
.} %This holds for all $\pi\in\Pi$, so we obtained $\min_{Q}~ \max_{\mul\geq 0}~ \Lagr_\mu(Q,\mul) = J(\pi^*)$.

Again, for any conjugate policy $\pi$, due to Lemma \ref{lem:lagr_dual}, the Lagrangian function has the dual form \eqref{lagr_dual} under the particular multiplier $\mul_\pi$, where \eqref{lagr_dual} is copied below for convenience of presentation:
\eq{
\Lagr_\pi(Q, \mul_\pi) = 
J(\pi) + \sum_{s\not\in\terminals} \sum_{a\in\AS} \mul_\pi(s,a) \* \Big( \max_{\bar{a}} Q(s,\bar{a}) - Q(s,a) \Big)
}
In the second term above, both $\mul_\pi(s,a)$ and $\max_{\bar{a}} Q(s,\bar{a}) - Q(s,a)$ are non-negative for any $Q$ and $\pi$, so it attains its minimum, which is zero, when $Q$ achieves complementary slackness with $\mul_\pi$. On the other hand, the first term above, i.e. $J(\pi)$, does not change with $Q$. So, the sum of the two terms, i.e. $\Lagr_{\pi}(Q,\mul_\pi)$, will attain its minimum when the second term is zero, that is, 
\eq[min_L]{
\min_{Q\in\Q}~ \Lagr_\pi(Q, \mul_\pi) = J(\pi) \quad,~~ \forall \pi
.}%$.

Now set the conjugate policy $\pi$ in \eqref{minimax_L} and \eqref{min_L} to the optimal policy $\mu$, as assumed in the theorem, we have 
\eq{
\max_{\mul\geq 0} \min_{Q\in\Q}~ \Lagr_\mu (Q,\mul) 
\geq \min_{Q\in\Q}~ \Lagr_\mu (Q,\mul_\mu) 
= J(\mu) = J(\pi^*) 
= \min_{Q} \max_{\mul\geq 0}~ \Lagr_\mu(Q,\mul)
.}

Because of the \emph{weak minimax duality} (which universally holds for any function), we also have 
$\max_{\mul} \min_{Q} \Lagr_\mu (Q,\mul) \leq \min_{Q} \max_{\mul} \Lagr_\mu (Q,\mul)$, which means the above inequality must actually be an equality, as desired.
\end{proof}

\subsection{Proof of Proposition \ref{prop:eq_condition}}
\label{sec:proof_eq_condition}

\begin{recap}[Proposition \ref{prop:eq_condition}]
Given a finite ELP, for any Q-function $Q$ and any policy $\pi$, let $\rho_\pi(s,a) = \rho_\pi(s) ~ \pi(a|s)$ and $\mul_\pi(s,a) = \rho_\pi(s,a) ~ \E_{\pi}[T]$, we have 
\eq{
		\Lagr_{\pi} (Q, \bar{\mul}) 	
\leq 	\Lagr_{\pi} (Q, \mul_\pi)
\leq 	\Lagr_{\pi} (\bar{Q}, \mul_\pi)
\quad,\quad \forall \bar{Q}, \bar{\mul}
}
if and only if
\begin{enumerate}[label=(\theenumi)]
\item \label{eq_constraint} 
$\Bopt Q(s,a) - Q(s,a) \leq 0$ \hspace{0.95in}\quad,\quad $\forall (s,a)\in\SS\times\AS$
\item \label{eq_slack_pi} 
$\rho_\pi(s,a) \* \Big( \Bopt Q(s,a) - Q(s,a) \Big) = 0$ \;~~\quad\quad,\quad $\forall (s,a)\in\SS\times\AS$
\item \label{eq_slack_q} 
$\rho_\pi(s,a) \* \Big( \max_{\bar{a}} Q(s,\bar{a}) - Q(s,a) \Big) = 0$ \quad,\quad $\forall s\not\in\terminals, a\in\AS$
\end{enumerate}
\end{recap}

\newcommand{\Qbar}{{ \bar{Q} }}
\newcommand{\pibar}{{ \bar{\pi} }}
\newcommand{\mulbar}{{ \bar{\mul} }}

The ``if'' part is straightforward: Condition \ref{eq_constraint} and \ref{eq_slack_pi} immediately gives $\Lagr_\pi(Q,\mul_\pi) = \E[\zeta\sim\pi][Q(S_T,A_T)] = \max_{\mulbar\geq0} \Lagr_\pi(Q,\mulbar)$. On the other hand, condition \ref{eq_slack_q} means that the second term in the dual-form Lagrangian \eqref{lagr_dual} is zero, so $\Lagr_\pi(Q,\mul_\pi) = J(\pi)$. By \eqref{min_L} in \ref{sec:proof_elp_minimax}, we have $\min_{\Qbar}~ \Lagr_\pi(\Qbar, \mul_\pi) = J(\pi)$, thus $\Lagr_\pi(Q,\mul_\pi) = \min_{\Qbar} \Lagr_\pi(\Qbar,\mul_\pi)$.

Now we prove the ``only if'' part, for which we resort to the general saddle-point condition: Under \emph{given} conjugate policy $\pi$, a $(Q,\mul)$ pair is a minimax saddle-point of function $\Lagr_\pi(Q,\mul)$ if and only if 
\begin{enumerate}[label=(\roman*)]
\item \label{eq_duality} $\displaystyle{ 
	\min_{\Qbar\in\Q} \max_{\mulbar\geq0} \Lagr_\pi(\Qbar,\mulbar) = 
	\max_{\mulbar\geq0} \min_{\Qbar\in\Q} \Lagr_\pi(\Qbar,\mulbar) =
	\Lagr_\pi(Q,\mul)
}$, 
\item \label{eq_minimax} $\displaystyle{ 
	Q \in \arg\min_{\Qbar\in\Q} \max_{\mulbar\geq0} \Lagr_\pi(\Qbar,\mulbar) 
}$,
\item \label{eq_maximin} $\displaystyle{ 
	\mul \in \arg\max_{\mulbar\geq0} \min_{\Qbar\in\Q} \Lagr_\pi(\Qbar,\mulbar) 
}$.
\end{enumerate}

By condition \ref{eq_minimax}, if $(Q,\mul_\pi)$ form a minimax equilibrium, then $Q$ must be a minimax solution of $\Lagr_{\pi}$. From \ref{sec:proof_Qstar_minimax} we know that such minimax Q-function must have $Q\geq \Bopt Q$, which gives condition \ref{eq_constraint}. 

By condition \ref{eq_duality}, and by \eqref{minimax_L} in \ref{sec:proof_elp_minimax}, we have $\Lagr_\pi(Q,\mul_\pi) = \min_{\Qbar\in\Q} \max_{\mulbar\geq0} \Lagr_\pi(\Qbar,\mulbar) = J(\pi^*) = \E[\zeta\sim\pi] [Q(S_T,A_T)]$. In other words, $\sum_{s,a} \lambda_\pi(s,a) ~ \big( \Bopt Q(s,a) - Q(s,a) \big)$, as the second term in $\Lagr_\pi(Q, \mul_\pi)$, must be zero in this circumstance. Because $\lambda_\pi(s,a) = \rho_\pi(s,a)\* \E[\zeta\sim\pi][T]$, where $\E[\zeta\sim\pi][T] > 0$, and further because $\Bopt Q \leq Q$ as just proved, the only way to make the term zero is to have $\rho_\pi(s,a) \* \big( \Bopt Q(s,a) - Q(s,a) \big) = 0$ at every $(s,a)$ pair, which gives condition \ref{eq_slack_pi}.

Moreover, since $\Lagr_\pi(Q,\mul_\pi) = \min_{\Qbar} \Lagr_\pi(\Qbar,\mul_\pi)$ as assumed, and $\Lagr_\pi(Q,\mul_\pi) = J(\pi)$ due to \eqref{min_L} in \ref{sec:proof_elp_minimax}, we know that the second term in the dual form of $\Lagr_\pi(Q,\mul_\pi)$ must be zero, that is,
$\sum_{s\not\in\terminals} \sum_{a\in\AS} \rho_\pi(s,a)\* \E[\zeta\sim\pi][T] \* \big( \max_{\bar{a}} Q(s,\bar{a}) - Q(s,a) \big) = 0$. Again, because $\E[\zeta\sim\pi][T] > 0$ and $\max_{\bar{a}} Q(s,\bar{a}) - Q(s,a) \geq 0$ for all $(s,a)$, the only possibility is to have $\rho_\pi(s,a)\* \big( \max_{\bar{a}} Q(s,\bar{a}) - Q(s,a) \big) = 0$ for each $(s,a) \in \SS\setminus\terminals \times \AS$ , which gives condition \ref{eq_slack_q}.

\section{Proofs of the Minimax-Maximin Symmetry Breaking (Section \ref{sec:maximin})}

\subsection{Proof of Lemma \ref{lem:Qstar_maximin}}
\label{sec:proof_Qstar_maximin}

\begin{recap}[Lemma \ref{lem:Qstar_maximin}]
In any finite ELP, $Q^*$ is an optimal solution of
\eqm{
\max_{Q} 	\quad  \Exp[\zeta\sim \pi]{Q(S_T,A_T)}  
\quad \text{s.t.} \quad  Q(s,a) \leq  \Bopt Q(s,a) 
~~,~~ \forall (s,a)
}
%\eqm[maxQ]{
%\max_{Q} 	\quad & 	\Exp[\zeta\sim \pi]{Q(S_T,A_T)}  \\
%\text{s.t.} \quad & 	Q(s,a) \leq  \sum_{s'}~ \max_{a'}~ \T(s'|s,a) \* \Big( R(s') + \ga(s') \* Q(s',a') \Big) 
%\quad,\quad \forall (s,a)
%}
for any conjugate policy $\pi$. Equivalently, $Q^* \in \arg \displaystyle{\max_{Q\in\Q}~\min_{\mul\geq 0}}~~ \Lagr_\pi(Q,\mul)$.
\end{recap}

The proof is by a symmetric argument with the one for Lemma \ref{lem:Qstar_minimax} (see \ref{sec:proof_Qstar_minimax}): Because $\B$ is monotonic, for $Q$ with $Q\leq \B Q$ we have $Q\leq \Bopt Q \leq (\Bopt)^2 Q \dots \leq Q^*$, thus $Q^*$ maximizes the objective $\E[\zeta\sim\pi][Q(S_T,A_T)]$ in a per-state-action manner.

\subsection{Proof of Theorem \ref{thm:Qopt}}
\label{sec:proof_Qopt}

\begin{recap}[Theorem \ref{thm:Qopt}]
In any finite ELP, for any conjugate policy $\pi$, let $\Qmax$ be an maximin Q-function with respect to the Lagrangian $\Lagr_\pi$, -- i.e. let $\Qmax$ be an optimal solution of \eqref{maxQ} -- then $\Qmax$ is an optimal Q-function, in the sense that $\Qmax$-greedy policy maximizes the total-reward objective \eqref{objective}.
\end{recap}

\newcommand{\abar}{{ \bar{a} }}
\newcommand{\SAnext}{(\SS\times\AS)_{\text{next}}}

\begin{proof}
First observe that 
\eq[Qmax_leq_Qstar]{
\Qmax(s,a) \leq Q^*(s,a) \quad,~~ \forall (s,a)\in\SS\times \AS
} 
which is because $\Qmax$, as a feasible solution of \eqref{maxQ}, has $\Qmax \leq \Bopt \Qmax \leq \Bopt\Bopt \Qmax \dots \leq Q^*$. 
Let $\mu$ be a $\Qmax$-greedy policy, so $\mu(a|s) > 0$ only if $\Qmax(s,a) = \max_{\bar{a}} \Qmax(s,\bar{a})$. Let $\SS_\mu\subseteq \SS$ be the set of states reachable by policy $\mu$. As described in the proof idea, we will focus on proving that 
\eq[Qmax_Qstar]{
\max_{a} \Qmax(s,\abar) = \max_a Q^*(s,a) \quad,\quad \forall s \in \SS_\mu \setminus \terminals
}
which would necessarily imply that 
\eq[Qmax_Qstar_argmax]{
\arg\max_a \Qmax(s,a) \subseteq \arg\max_a Q^*(s,a) \quad,\quad \forall s \in \SS_\mu \setminus \terminals
.}
Note that \eqref{Qmax_Qstar} entails \eqref{Qmax_Qstar_argmax} because, by \eqref{Qmax_leq_Qstar}, for any action $a$ sub-optimal to $Q^*$, it can only have even lower Q-value in $\Qmax$, with $\Qmax(s,a) \leq Q^*(s,a) < \max_\abar Q^*(s,\abar) = \max_\abar \Qmax(s,\abar)$, where $\Qmax(s,a) < \max_\abar \Qmax(s,\abar)$ (for the $Q^*$-suboptimal action $a$) guarantees that such an $a$ cannot be $\Qmax$-optimal either. \eqref{Qmax_Qstar_argmax} guarantees that $\mu$, as a $\Qmax$-greedy policy, will only choose $Q^*$-optimal actions at every non-terminal state it may encounter since time $t \geq 1$. Such a $\mu$ is equivalently to a $Q^*$-greedy policy, thus is also an optimal policy. Note that $\mu$'s choices on terminal states does not matter here as state-transitions in terminal steps are action-agnostic, due to the ELP condition.

Now, to prove \eqref{Qmax_Qstar}, we first prove the following induction rule:
\begin{proposition}
\label{prop:Qmax_induction}
Under the context of Theorem \ref{thm:Qopt}, let $(s,a)$ be an arbitrary state-action pair, and let 
\eq{
\SAnext \define \{ (s',a')\in\SS\times\AS :~ s'\not\in\terminals \texttt{~and~} \T(s'|s,a) \* \mu(a'|s') > 0 \}
} 
denote the set of all the non-terminal $(s',a')$ pairs that can directly follow $(s,a)$ under the $\Qmax$-greedy policy $\mu$, then
\eqm[Qmax_induction]{
&~~\Qmax(s,a) = Q^*(s,a) \\ \Rightarrow 
&~~\Qmax(s',a') = Q^*(s',a') \texttt{~~and~} \max_{\bar{a}} \Qmax(s',\bar{a}) = \max_{\bar{a}} Q^*(s',\bar{a})
~~,~~ \forall (s',a') \in \SAnext
} 
\end{proposition}
\begin{proof}
Because $\Qmax \leq Q^*$ and $\Qmax \leq \Bopt \Qmax$, we have
\eqm[Qopt_1]{
\Qopt(s, a) 
&\leq \B \Qopt(s, a) \\
&= \E_{S'}[R(S')] + \sum_{s'\in\SS} P(s'|s,a) \* \gamma(s') \* \max_{\abar\in\AS}~ \Qopt(s',\abar) \\
&\leq \E_{S'}[R(S')] + \sum_{s'\in\SS} P(s'|s,a) \* \gamma(s') \* \max_{\abar\in\AS}~ Q^*(s',\abar) \\
&= \Bopt Q^*(s, a) \\
&= Q^*(s, a) 
}
The premise $\Qmax(s,a) = Q^*(s,a)$ in the induction rule \eqref{Qmax_induction} entails that the two inequality signs in above must be equality, among which the second one -- i.e. the one leading \eqref{Qopt_1} -- can be equality only if 
\eq[Qopt_3]{
\sum_{s'\in\SAnext} P(s'|s,a) \* \max_{\abar\in\AS}~ \Qopt(s',\abar) 
= \sum_{s'\in\SAnext} P(s'|s,a) \* \max_{\abar\in\AS}~ Q^*(s',\abar)
}
where $s'\in\SAnext$ is a slight abuse of notation which means $s'$ shows up in $\SAnext$ (in the form of a $(s',a')$ pair, with some $a'$), or equivalently, $s'\in\SAnext$ means that $s'\not\in\terminals$ and $\T(s'|s,a)>0$).

Now observe that for \eqref{Qopt_3} to hold, the only possibility is that
\eq[Qopt_4]{
\max_{\abar\in\AS}~ \Qopt(s',\abar) = \max_{\abar\in\AS}~ Q^*(s',\abar) \quad,\quad \forall s'\in\SAnext
} 
as otherwise for those $s'$ on which \eqref{Qopt_4} do not hold, it can only be $\max_{\abar\in\AS}~ \Qopt(s',\abar) < \max_{\abar\in\AS}~ Q^*(s',\abar)$ (because $\Qopt \leq Q^*$); those $s'$ must all have positive weights in \eqref{Qopt_3} (by definition of $\SAnext$), and thus will cause a real loss at the LHS of \eqref{Qopt_3} (and importantly, no other state in $\SAnext$ could claim a ``gain'' to compensate this loss, again because $\Qopt \leq Q^*$). \eqref{Qopt_4} is exactly the second consequence in the induction rule \eqref{Qmax_induction}.

Next, to prove $\Qmax(s',a') = Q^*(s',a')$ for all $(s',a')\in \SAnext$, the first consequence in the induction rule \eqref{Qmax_induction}), we notice that for any of such $(s',a')$ we have
\eq[Qopt_5]{
\max_{\abar\in\AS}~ \Qopt(s',\abar) = \Qopt(s',a') \leq Q^*(s',a') \leq \max_{\abar\in\AS}~ Q^*(s',\abar)
}
in which $\max_{\abar\in\AS}~ \Qopt(s',\abar) = \Qopt(s',a')$ is because $a'$ is by definition a $\Qmax$-greedy action under $s'$, and $\Qopt(s',a') \leq Q^*(s',a')$ is (once again) because $\Qopt \leq Q^*$.
 
By \eqref{Qopt_4} we know that the two ends of \eqref{Qopt_5} actually equal to each other, so the inequalities in between must also be equality, and in particular $\Qopt(s',a') = Q^*(s',a')$, as desired.
\end{proof}

Proposition \ref{prop:Qmax_induction} enables us to prove \eqref{Qmax_Qstar} by induction (which is enough to prove the whole theorem, as argued above). Specifically, because both $\Qmax$ and $Q^*$ are optimal solutions of \eqref{maxQ}, and because the objective in \eqref{maxQ} is a distribution over only the terminal states, it follows that $\Qmax$ and $Q^*$ must be equal on at least one terminal state $s_\perp$. Starting from this terminal state $s_\perp$ -- as well as an arbitrary action $a_\perp$ under it -- we have $\Qmax(s_\perp,a_\perp) = Q^*(s_\perp,a_\perp)$, thus by the induction rule of Proposition \ref{prop:Qmax_induction} we obtain $\max_\abar \Qmax(s',\abar) = \max_\abar Q^*(s',\abar)$ and $\Qmax(s',a')=Q^*(s',a')$ for all $(s',a')$ in the $\SAnext$ set with respect to $(s,a) = (s_\perp,a_\perp)$; the latter enables us to expand the induction proof to all non-terminal states that are reachable by $\mu$.
\end{proof}

\subsection{An counter-example showing that a minimax Q-function can be sub-optimal in ELPs}
In this subsection we elaborate more about the counter-example as illustrated by Figure \ref{fig:qexample} in Section \ref{sec:maximin} (the figure is copied above). In this ELP, $\SS=\{0,1,2,3,4,5\}$, $\AS=\{1,2,3\}$. State $4$ and $5$ are terminal states, from which any action leads to state $0$. Choosing action $1,2,3$ under state $0$ deterministically transits to state $1,2,3$, respectively. All actions under state $1$ lead to state $4$, and all actions under state $2$ and $3$ lead to state $5$. The agent only receives non-zero rewards at terminal states, with $R(4)=1$, $R(5)=2$. The initial state at time $0$ is set to state $4$ (i.e. $\rho_0(s)>0$ only if $s=4$). 

The Bellman fixed-point $Q^*$ for this ELP is as follows:
\squishlist
%\item $Q^*(0,1) = 1~~~$ ~,\hspace{0.15in} $Q^*(0,2) = 2~~~$ ~,\hspace{0.15in} $Q^*(0,3) = 2~~~$ 
%\item $Q^*(1,1) = 1~~~$ ~,\hspace{0.15in} $Q^*(1,2) = 1~~~$ ~,\hspace{0.15in} $Q^*(1,3) = 1~~~$
%\item $Q^*(2,1) = 2~~~$ ~,\hspace{0.15in} $Q^*(2,2) = 2~~~$ ~,\hspace{0.15in} $Q^*(2,3) = 2~~~$  
\item $Q^*(0,1) = 1$,~~ $Q^*(0,2) = 2$,~~ $Q^*(0,3) = 2$ 
\item $Q^*(1,a) = 1$,~ $\forall a$ 
\item $Q^*(2,a) = Q^*(3,a) = 2$,~ $\forall a$
\item $Q^*(4,a) = Q^*(5,a) = 2$,~ $\forall a$
\squishend
An optimal policy of this ELP should only choose action $2$ or $3$, but not action $1$, under state $0$.

\begin{figure}[t]
\centering
\includegraphics[width=0.35\textwidth]{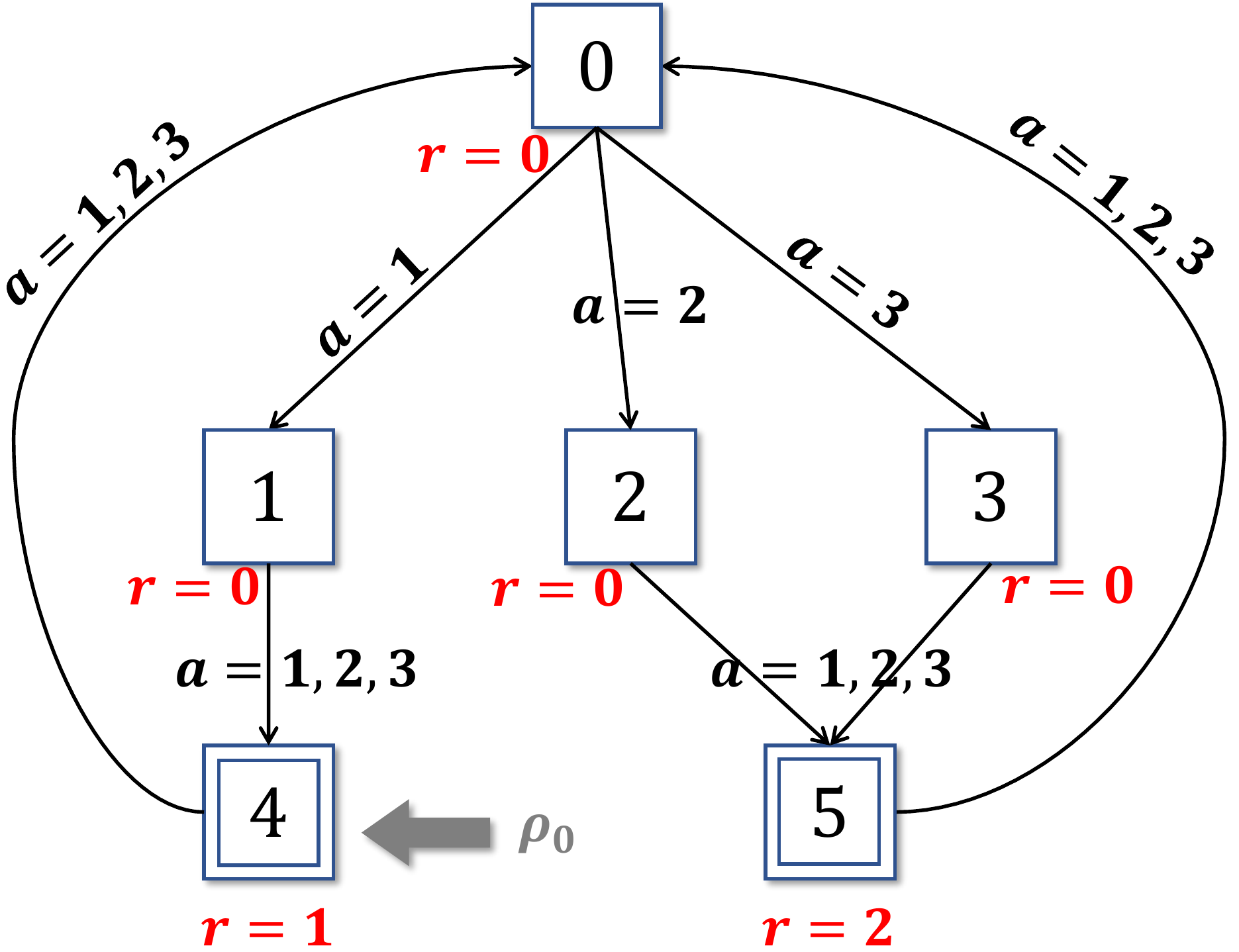}
\caption{A copy of Figure \ref{fig:qexample}}%, which shows an example of ELP for studying saddle points of the $Q$-form Lagrangian.}
\end{figure}

For minimax Q-functions, denoted by $\Qmin$, as they are optimal solutions of \eqref{minQ}, we have
\eq{
\Qmin(4,a) \quad=\quad 2 \quad\geq\quad \max 
\begin{cases}
~~ \Qmin(0,1) \quad\geq\quad \max_a \Qmin(1,a) \quad\geq\quad 1 \\
~~ \Qmin(0,2) \quad\geq\quad \max_a \Qmin(2,a) \quad\geq\quad 2 \\
~~ \Qmin(0,3) \quad\geq\quad \max_a \Qmin(3,a) \quad\geq\quad 2 
\end{cases}
} 
The above minimax condition only imposes tight bounds for $Q^*$-optimal actions (e.g. action $2$ and $3$ under state $0$), but leaves ``flexibility'' for actions sub-optimal to $Q^*$ (e.g. action $1$ under state $0$) as well as for \emph{all} state-action pairs that follow an sub-optimal action (e.g. all actions under state $1$). The consequence is that even constant value function $Q_{\min}(s,a)\equiv 2$ can be a minimax value function in this example, which is clearly sub-optimal, as discussed in Section \ref{sec:maximin}.

In contrast, for maximin Q-functions, denoted by $\Qmax$, they are optimal solutions of \eqref{maxQ}, thus
\eq{
\Qmax(4,a) \quad=\quad 2 \quad\leq\quad \max 
\begin{cases}
~~ \Qmax(0,1) \quad\leq\quad \max_a \Qmax(1,a) \quad\leq\quad 1 \\
~~ \Qmax(0,2) \quad\leq\quad \max_a \Qmax(2,a) \quad\leq\quad 2 \\
~~ \Qmax(0,3) \quad\leq\quad \max_a \Qmax(3,a) \quad\leq\quad 2 
\end{cases}
} 
We see that the maximin condition manages to imposes tight bounds for \emph{at least one} $Q^*$-optimal action while at the same time can enforce \emph{all} $Q^*$-sub-optimal actions to be still suboptimal to $\Qmax$. For example under state $0$, $\Qmax(0,1)$ cannot exceed $1$, while either $\Qmax(0,2)$ or $\Qmax(0,3)$ needs to be tight (i.e. $\Qmax(0,3)=2$, or $\Qmax(0,2)=2$, or both) so as to keep the maximum of the three no less than $2$, as required.

Note that for both $\Qmin$ and $\Qmax$, the Lagrangian multiplier $\mul$ that forms equilibrium/saddle points with each of them (resp.) may not encode a policy, in general. In this example, for the constant $\Qmin$, we have $2 = \Qmin(0,1) = 0 + 1 \* \Qmin(1,a) > 1+0\* \Qmin(4,a) = 1$, so $\Qmin(1,a) > \Bopt \Qmin(1,a) = 1$ for all $a$, in which case its equilibrium multiplier $\mul$ has to have $\mul(1,a)=0$ for all $a$, due to the equilibrium condition \ref{eq_slack_pi} proved in Proposition \ref{prop:eq_condition}. Such an ``all-zero'' $\mul$ (on state $1$) cannot be normalized into a policy. 

Similarly, for the following specific maximin Q-function
\squishlist
\item $\Qmax(4,a) = \Qmax(5,a) = 2$
\item $\Qmax(0,1) = \Qmax(1,a) = 1$
\item $\Qmax(0,2) = \Qmax(2,a) = 2$
\item $\Qmax(0,3) = 1$
\item $\Qmax(3,a) = 1.5$ 
\squishend
we have $1 = \Qmax(0,3) < 0 + 1\* \Qmax(3,a) < 2 + 0 \* \Qmax(5,a) = 2$, so $\Qmax(3,a) < \Bopt \Qmax(3,a) = 2$ again for all $a$, so the multiplier corresponding to this $\Qmax$ still needs to be all zero at state $3$ due to the complementary slackness condition, thus cannot be normalized at state $3$.

\subsection{A counter-example showing that a minimax $V$-function can be sub-optimal in discounted-MDPs}
\label{sec:example}
Moreover, the problems with minimax points of the Lagrangian, as demonstrated above, are not limited to $Q$-functions or to ELPs only, but seem to be fundamental issues rooted from the minimax structure. To see this, consider the \emph{discounted-MDP} as shown in Figure \ref{fig:dmdp} below. 

To be strictly aligned with the related literature~\cite{2018:LP_RL,2017:LP_RL}, the rewards 
%\begin{wrapfigure}{r}{0.5\textwidth}
\begin{figure}[t]
\centering
\includegraphics[width=0.4\textwidth]{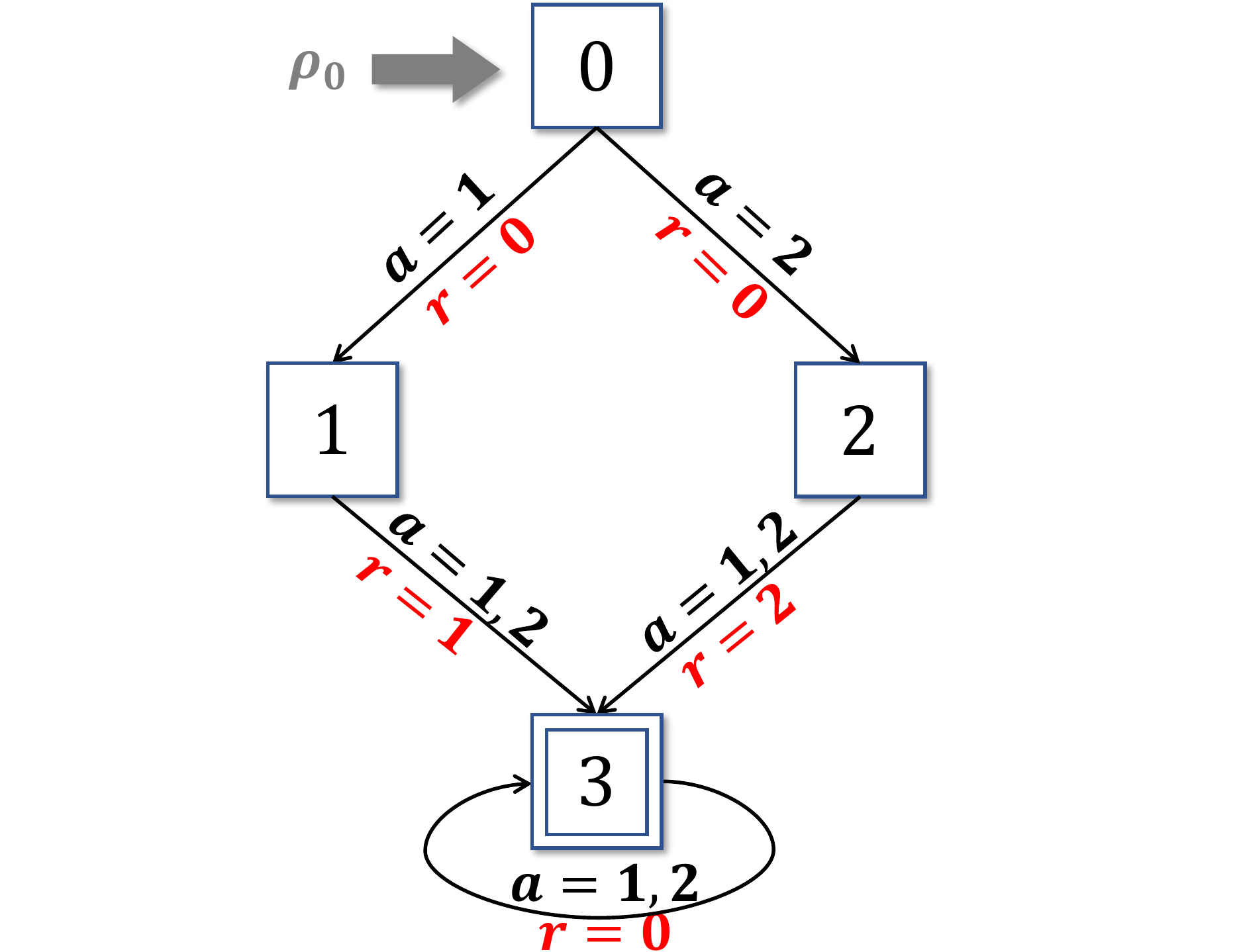}
\caption{An example of discounted-MDP for studying saddle points of the $V$-form Lagrangian.}
\label{fig:dmdp}
\end{figure} 
%\end{wrapfigure}
are assigned to state-action pairs in this example. In this discounted-MDP, $\SS=\{0,1,2,3\}$, $\AS=\{1,2\}$. The initial state is set to state $0$ (i.e. $\rho_0(s)=1$ if $s=0$, otherwise $\rho_0(s)=0$). From state $0$, taking action $1,2$ will deterministically goes to state $1,2$, respectively, with zero reward obtained in this step. From state $1$, any action leads to the absorbing state $3$, with reward $R(1,a)=1$ obtained (for all $a$). From state $2$, any action leads to the same absorbing state $3$, but with reward $R(2,a)=2$ obtained (for all $a$). The absorbing state $3$ will loop into itself forever, with zero reward obtained. An optimal policy in this discounted-MDP should choose action $2$, not action $1$, under state $0$. The discounting constant $\gamma$ is set to $0.5$, so the optimal discounted-reward performance is $V^*(0) = 1$.

\newcommand{\Vmin}{{ V_{\min} }}
The $V$-form Lagrangian of the discounted-MDP above is:~\cite{2018:LP_RL, 2017:LP_RL}
\eq[v_lagr]{
(1-\ga) \* \Exp[S_0\sim\rho_0]{V(S_0)} + \sum_{(s,a)\in \SS\times \AS} \lambda(s,a) \* \Big( R(s,a) + \ga \* \Exp[S'\sim\T(s,a)]{V(S')} - V(s) \Big)
} 
A minimax V-function $\Vmin$ of the V-form Lagrangian \eqref{v_lagr} is an optimal solution of the Linear Programming problem \eqref{lp_v}, which has inspired some recently proposed RL algorithms (see Section \ref{sec:lagrangian}). In this example, the LP can be more explicitly written as
\eqm{
\min_{V} \quad & 		0.5 \* V(0)  \\
\text{s.t.} \quad & 	V(0) \geq 0.5 \* V(1) \\
\quad & 				V(0) \geq 0.5 \* V(2) \\
\quad & 				V(1) \geq 0.5 \* V(3) + 1 \\
\quad & 				V(2) \geq 0.5 \* V(3) + 2 \\
\quad & 				V(3) \geq 0.5 \* V(3) 
}
for the LP above, a possible optimal solution is:
$%\eq{
\Vmin(0) = 1,~~%\quad
\Vmin(1) = 2,~~%\quad
\Vmin(2) = 2,~~%\quad
\Vmin(3) = 0
$, %}
which assigns the same value to action $1$ and $2$ under state $0$, thus is not an optimal V-function.

\section{Lagrangian Minimization for Machine Translation}

\subsection{The LAMIN1 Algorithm}
\label{sec:lamin1}

As mentioned in Section \ref{sec:mt}, the idea of LAMIN1 is to minimize the smoothed Lagrangian $\Lagr_\mu^\beta$ (with a small yet definite $\beta$) based on an unbiased gradient estimator of \eqref{lamin1}. For convenience, we copy \eqref{lamin1} below: 
\eq{
	\Lagr_\mu^\beta(Q(\w),\mul_\mu) = \E_\mu [Q(S_T, A_T;\w)] + \E_\mu [T] \*
	\E[S,A,S'\sim\rho_\mu]~ \Exp[A'\sim \pi^\beta_{Q(\w)}(S')]{\delta(S,A,S',A';\w)}
}
where $\pi^\beta_{Q(\w)}(a|s) \define \frac{exp\big( Q(s, a;\w) / \beta \big)}{\sum_b exp\big( Q(s, b;\w) / \beta \big)}$ is the Boltzmann distribution with temperature $\beta$, and $\delta(s,a,s',a';\w) \define R(s') + \gammaEPI(s') Q(s',a';\w)-Q(s,a;\w)$ is the temporal-difference error. 
Algorithm \ref{algo:lamin1} gives the pseudo-code of such an algorithm. 

\begin{algorithm}[h]
\caption{The LAMIN1 algorithm.}
\label{algo:lamin1}
\begin{algorithmic}
\STATE {\bfseries Input:} A piece of rollout data made by an optimal policy, in the form of $\{s_t,a_t,r_t\}_{0,1,\dots,n}$, in which $t=T_1,T_2,\dots,T_k$ are termination steps; A parametric $Q(\w)$ model with initial weight $\w_0$; Learning rate $\alpha$; Boltzmann temperature $\beta$.
%\BlankLine
\FOR{gradient update $i = 0, 1, 2, \dots$}
\STATE $\Delta \w \gets 
\frac{1}{k} \sum_{t=\{T_1 \dots T_k\}} \nabla_{\w} Q(s_t,a_t;\w) ~\Big|_{\w=\w_{i}} ~+~$ 

~~~~~~~~~~~~~~$\frac{n}{k} \* \frac{1}{n} \sum\limits_{t=0}^{n-1} \gammaEPI(s_{t+1})~ 
\nabla_{\w} \Big(  \sum_a \pi^\beta_{\w}(s_{t+1},a) \* Q(s_{t+1},a;\w) \Big)
- \nabla_{\w} Q(s_{t},a_{t};\w)  ~\Big|_{\w=\w_{i}}$

\STATE $\w_{i+1} \gets \w_{i} - \alpha \* \Delta \w$ 
\ENDFOR
%\BlankLine
\STATE {\bfseries Output:} A $Q(\w)$-greedy policy.
\end{algorithmic}
\end{algorithm}

Algorithm \ref{algo:lamin1} samples the stationary distribution $\rho_\mu$ in \eqref{lamin1} by averaging over the rollout data of $k$ episodes, which is unbiased thanks to the ergodicity of episodic learning~\cite{2020:bojun}. The average episode length $\E_\mu [T]$ is estimated by $n/k$, the total number of steps in the data divided by the number of episodes. The overall algorithm is thus a standard unbiased SGD procedure, which shares the generic convergence property of all SGD procedures (i.e. convergence to local minimum of \eqref{lamin1} is guaranteed under properly annealed learning rate~\cite{2016:DL}). We also remark that the reward term $R(s')$ in the smoothed Lagrangian (more specifically, in the TD-error of \eqref{lamin1}) is ``differentiated out'' in LAMIN1 because it is independent of the model parameter $\w$. As a result, Algorithm \ref{algo:lamin1} does not use the reward data at all. 

As an implementation trick, the gradient computation in Algorithm \ref{algo:lamin1}, particularly for $\nabla_{\w} \Big(  \sum_a \pi^\beta_{\w}(s_{t+1},a) \* Q(s_{t+1},a;\w) \Big)$, can conveniently run on automatic differentiation libraries such as PyTorch by utilizing the following fact (for brevity and clarity, we omit the argument $s_{t+1}$ in $Q$ and $\pi$ in the following):
\eqm{
	\nabla \pi^\beta_{\w}(a) 
&= 	\nabla \exp \Big( \log \frac{e^{Q(a;\w)/\beta}}{\sum_b e^{Q(b;\w)/\beta}} \Big) 
= 	\pi^\beta_{\w}(a) \* \nabla 
	\Big( \log \frac{e^{Q(a;\w)/\beta}}{\sum_b e^{Q(b;\w)/\beta}} \Big) \\
&= 	\pi^\beta_{\w}(a) \* \Big( \nabla Q(a;\w)/\beta ~-~ 
	\nabla \log \sum\nolimits_b e^{Q(b;\w)/\beta} 
	\Big) \\
&= 	\pi^\beta_{\w}(a) \* \Big( \nabla Q(a;\w)/\beta ~-~ 
	\frac{\sum\nolimits_b e^{Q(b;\w)/\beta} \* \nabla Q(b;\w)/\beta}{\sum\nolimits_c e^{Q(c;\w)/\beta}} 
	\Big) \\
&=	\frac{1}{\beta} \* \Big( \pi^\beta_{\w}(a)~\nabla Q(a;\w) - 
	\pi^\beta_{\w}(a)~\sum\nolimits_b \pi^\beta_{\w}(b)~\nabla Q(b;\w) \Big)
}
and so
\eqm{
&	\nabla \Big(  \sum_a \pi^\beta_{\w}(a) \* Q(a;\w) \Big) \\
=~& 	\sum_a \pi^\beta_{\w}(a) \* \nabla Q(a;\w) + 
	\sum_a Q(a;\w) \* \nabla \pi^\beta_{\w}(a) \\
=~&	\sum_a \pi^\beta_{\w}(a) \* \nabla Q(a;\w) ~+~ 
	\frac{1}{\beta} \* \sum_a \pi^\beta_{\w}(a)~Q(a;\w)~\nabla Q(a;\w) \\
&	~-~ \frac{1}{\beta} \* 
	\Big( \sum_a \pi^\beta_{\w}(a)~Q(a;\w) \Big) \* 
	\Big( \sum_a \pi^\beta_{\w}(a)~\nabla Q(a;\w) \Big) \\
} 

Finally, we notice that the $\beta$-smoothing trick used in LAMIN1 is more than just an approximation heuristic, but may potentially play a role in correcting (to some extent) the sub-optimality bias of minimax Q-functions as discussed in Section \ref{sec:maximin}. Specifically, as an inherent weakness, the original Lagrangian function $\Lagr_{\mu}(Q,\mul_\mu)$ cannot distinguish minimax Q-functions that are optimal from minimax Q-functions that are sub-optimal, as the Lagrangian attains its global minimum in both cases. In contrast, the smoothed Lagrangian $\Lagr^\beta_{\mu}(Q,\mul_\mu)$ tends to reach lower value at optimal minimax-Q-functions than at sub-optimal minimax-Q-functions. In fact, it can be proved that the sub-optimality bias is completely resolved by the $\beta$-smoothing trick in tabular settings.

\begin{figure}[t]
\centering
\includegraphics[width=0.4\linewidth]{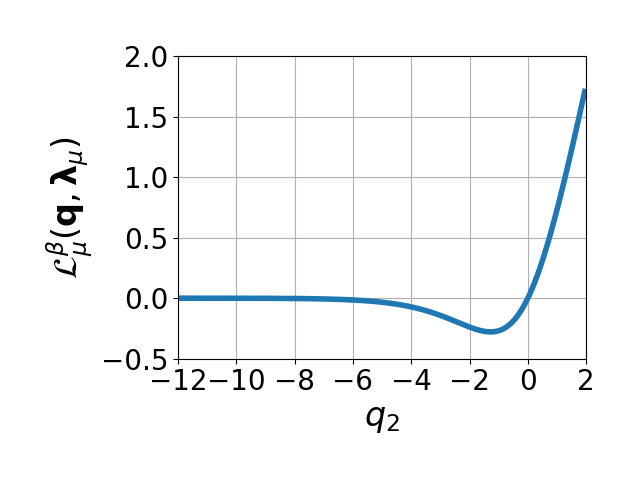}
\caption{An illustration of the smoothed Lagrangian ($\beta=1$) in the tabular setting with one state and two actions, with $q_1$ anchored at $0$. The Lagrangian is not convex, yet it has only one local minimum, which is attained around $q_2=-1.3$.}
\label{fig:degenerate_q}
\end{figure}

As an illustration, consider the special case where there is only one non-terminal state, under which the agent can only choose between two actions, $1$ and $2$, and suppose action $1$ is the truly optimal. In this simple case, a Q-function can be represented by a tuple $(Q_1,Q_2)$, which indicating the value of action $1$ and $2$ respectively. A sub-optimal minimax-Q-function may have $Q_1=Q_2$, as we showed in Section \ref{sec:maximin}; in this case the smoothed Lagrangian $\Lagr^{\beta=1}_{\mu}(Q,\mul_\mu)$ would still equal $J(\mu)$, which can be seen from the dual form of the smoothed Lagrangian:
\eq{
\Lagr^{\beta=1}_\mu (Q, \mul_\mu) = 
J(\mu) + 1 \* \Big( 
(\frac{e^{Q_1} \* Q_1}{e^{Q_1}+e^{Q_2}} + \frac{e^{Q_2} \* Q_2}{e^{Q_1}+e^{Q_2}}) - Q_1 \Big)
}
On the other hand, an optimal minimax-value can make $\Lagr^{\beta=1}_\mu (Q, \mul_\mu)$ lower than $J(\mu)$. For example, when $J(\mu)=Q_1=0$, we have $\Lagr^{\beta=1}_\mu (Q, \mul_\mu) = \frac{e^{Q_2}}{1+e^{Q_2}} \* Q_2$, which attains $-0.28$ at $Q_2=-1.3$. In other words, $\Lagr^{\beta=1}_\mu (Q, \mul_\mu) < J(\mu) = 0$ under the optimal minimax-Q-function $(0,-1.3)$, which is thus distinguished from the sub-optimal minimax-Q-function $(0,0)$ in smoothed Lagrangian minimization as LAMIN1 does. See Figure \ref{fig:degenerate_q} for the shape of the smoothed Lagrangian in this case. Notice that the Lagrangian is \emph{not} convex, despite the unique local minimum (which means LAMIN1 will converge to the global optimum in this case).

\subsection{The LAMIN2 Algorithm}
\label{sec:lamin2}

The idea of LAMIN2 is to minimize the original Lagrangian function $\Lagr_{\mu}$ based on the ``local'' gradient estimator \eqref{lamin2'}. The consistency of the LAMIN2 estimator (with respect to $\nabla \Lagr_{\mu}$) is asserted by Proposition \ref{prop_lagrad} in Section \ref{sec:lamin}, which we now provide a proof. 

Proposition \ref{prop_lagrad} says that \eqref{lamin2'} becomes equality if the Q-model is unimodal. For convenience we copy \eqref{lamin2'} below (with $[\Bopt Q-Q]$ and $\delta$ explicitly expanded, and the approximation sign replaced):
\eqm[lamin2'']{
&	\nabla_{\w}~ \Exp[S'\sim P(s,a)]{R(S') + \gammaEPI(S')~ \max_{a'} Q(S',a';\w)} - Q(s,a;\w)  ~\big|_{\w=\w_t} \\
=~&	\E[S'\sim P(s,a)]~ \Exp[A'\sim \pi_{Q(\w_t)}(S')]{\nabla_{\w}~
	\Big(~ (R(S') + \gammaEPI(S') Q(S',A';\w)-Q(s,a;\w) ~\Big)
	} ~\big|_{\w=\w_t}
}
Clearly, \eqref{lamin2''} holds if and only if equation \eqref{equiv_grad_1} in the following proposition holds.

\begin{proposition}
\label{prop_lagrad2}
Let $\AS$ be a finite action space, and let $Q(s,a;\w)$ be a differentiable parametric Q-model that suggests a single best action $a_{\max}(\w^*) \define \arg\max_{a\in\AS}~ Q(s,a;\w^*)$ when evaluating the actions under a given state $s$ with a given parameter vector $\w^*$, then
\eq[equiv_grad_1]{
	\nabla_{\w}~ \max_{a\in\AS}~ Q(s,a;\w) ~\Big|_{\w=\w^*}
	%\nabla_{\w}~ \Exp[a\sim\pi_{\max}(s;\w)]{Q(s,a;\w)} ~\Big|_{\w=\w^*} 
= 	\Exp[a\sim\pi_{\max}(s;\w^*)]{\nabla_{\w}~ Q(s,a;\w)} ~\Big|_{\w=\w^*}	
}
where $\pi_{\max}(\w^*)$ denote the $Q(\w^*)$-greedy policy.
\end{proposition}
\begin{proof}
We will prove that for any component of $\w$,
\eq[equiv_grad_2]{
	\frac{\partial}{\partial w_i}~ \max_{a\in\AS}~ Q(s,a;\w) ~\Big|_{\w=\w^*} 
= 	\frac{\partial}{\partial w_i}~ Q(s, a_{\max}(\w^*); \w) ~\Big|_{\w=\w^*}
}
which readily gives \eqref{equiv_grad_1}.

To prove \eqref{equiv_grad_2}, notice that 
\eq{
	\frac{\partial}{\partial w_i}~ \max_{a\in\AS}~ Q(s,a;\w) ~\Big|_{\w=\w^*} 
= 	\lim\limits_{\Delta\ra 0} \frac{Q(s, a_{\max}(\w^*+\Delta); \w^*+\Delta) - Q(s, a_{\max}(\w^*); \w^*)}{\Delta}
}
When there are a finite number of possible actions, there must be a non-zero gap between the best action and the second best action under $\w^*$. On the other hand, since $Q$ is differentiable, the change of action-values becomes infinitely small as $\Delta \ra 0$, thus the order between the two best actions will remain the same for small enough $\Delta$, that is, $a_{\max}(\w^*+\Delta) = a_{\max}(\w^*)$ when $\Delta \ra 0$. Therefore,
\eqm{ 
& 	\lim\limits_{\Delta\ra 0} \frac{
	Q(s, a_{\max}(\w^*+\Delta); \w^*+\Delta) - Q(s, a_{\max}(\w^*); \w^*)
	}{\Delta} \\
=~&	\lim\limits_{\Delta\ra 0} \frac{
	Q(s, a_{\max}(\w^*); \w^*+\Delta) - Q(s, a_{\max}(\w^*); \w^*)
	}{\Delta} \\ 
=~&	\frac{\partial}{\partial w_i}~ Q(s, a_{\max}(\w^*); w) ~\Big|_{\w=\w^*}
}
\end{proof}

Note that when $Q(\w)$ is a sophisticated real-valued function, such as a deep neural network with scalar output, the chance that two actions have \emph{precisely} the same value under $Q(\w)$ should be rare. Also, continuous action space can be densely quantized into a finite action space with \emph{arbitrarily} small quantizing error, thus \eqref{equiv_grad_1} should still approximately hold even for continuous action spaces.

Algorithm \ref{algo:lamin2} gives the pseudo-code of LAMIN2 in a further generalized form, where the greedy policy $\pi_{\w_i}$ is replaced with the Boltzmann policy $\pi^\beta_{\w_i}$. As mentioned in Section \ref{sec:mt}, higher temperatures, such as $\beta=1.0$, can slightly improve performance in our WMT experiment (for $0.5 - 0.8$ BLEU score).

\begin{algorithm}[t]
\caption{The LAMIN2 algorithm.}
\label{algo:lamin2}
\begin{algorithmic}
\STATE {\bfseries Input:} A piece of rollout data made by an optimal policy, in the form of $\{s_t,a_t,r_t\}_{0,1,\dots,n}$, in which $t=T_1,T_2,\dots,T_k$ are termination steps; A parametric $Q(\w)$ model with initial weight $\w_0$; Learning rate $\alpha$; Boltzmann temperature $\beta$.
%\BlankLine
\FOR{gradient update $i = 0, 1, 2, \dots$}
\STATE $\Delta \w \gets 
\frac{1}{k} \sum_{t=\{T_1 \dots T_k\}} \nabla_{\w} Q(s_t,a_t;\w) ~\Big|_{\w=\w_{i}} ~+~$ 

~~~~~~~~~~~~~~$\frac{n}{k} \* \frac{1}{n} \sum\limits_{t=0}^{n-1} \gammaEPI(s_{t+1})~ 
\Big(  \sum_a \pi^\beta_{\w_i}(s_{t+1},a) \nabla_{\w} Q(s_{t+1},a;\w) \Big)
- \nabla_{\w} Q(s_{t},a_{t};\w)  ~\Big|_{\w=\w_{i}}$

\STATE $\w_{i+1} \gets \w_{i} - \alpha \* \Delta \w$ 
\ENDFOR
%\BlankLine
\STATE {\bfseries Output:} A $Q(\w)$-greedy policy.
\end{algorithmic}
\end{algorithm}

\subsection{An Episodic Learning Formulation of Machine Translation}
\label{sec:elp_mt}

Many AI tasks are \emph{sequence generation} problems, where we are given a context $\X$, and are then asked to generate a sequence $\Y = \seq{\bos, y_1 \dots \y{L}, \eos}$ using \emph{tokens} chosen from a given token space. Machine Translation (MT) is an example of such tasks, where the token space is the vocabulary of a target language, and the context $\X$ is a sentence (or a sequence of sentences) in a source language. 
The choice of each token $y_t$
%\footnote{$y_t \in \vocab$ for $1\leq t \leq L$. $\bos$ and $\eos$ are special tokens indicating the beginning and ending of a sequence. Without loss of generality we assume $\bos, \eos \in \vocab$ too.} 
is conditioned on $\X$ and on the partial output $\Yp{t} \define \seq{\bos, y_1 \dots y_{t-1}}$. In particular, $\Yp{1}=\seq{\bos}$, and $\Yp{L+1}=\seq{\bos,y_1\dots \y{L}}$, conditioned on which the algorithm will generate the first token $y_1$ and the last token $\y{L+1} \define \eos$, respectively. 

The ELP formulation exactly captures the real-world MT tasks as described above. In the MT context, an episode is the translation of a given sentence. The first episode effectively starts with $S_1 = (\X[1],\bos)$ where $\X[1]$ is a full source sentence. A learning agent then chooses a token $A_1 =\y[1]{1} \in \vocab[\text{target}] \cup \{\eos\}$, after which the environment state transits, deterministically, to $S_2 = (\X[1], \Yp[1]{2}) = (\X[1], \bos, \y[1]{1})$. The agent keeps generating actions $A_t = \y[1]{t}$ under each $S_t=(\X[1],\Yp[1]{t})$ until it outputs $A_{T-1} = \eos$  at some step $T-1$, leading to terminal state $S_{T} = (\X[1],\Y[1]) = (\X[1], \bos, \y[1]{1} \dots \y[1]{T-2}, \eos)$. The agent will then make a normal action $A_T$ as in previous steps, which however makes no effect other than resetting the environment into $S_{T+1} = (\X[2],\bos)$ from which the second translation episode begins. The process goes on episode after episode, generating a (theoretically infinite) sequence of translations $(\X[1],\Y[1]), (\X[2],\Y[2]), \dots$, which collectively serve as the training data for the agent to learn better translation policies.

An episode of length $T$ as above results in a translation sentence $Y=(\bos,y_1 \dots y_L, \eos)$ which contains $L = T-2$ ``normal'' tokens from $\vocab[\text{target}]$. As a common and necessary practice, most real-world MT systems impose a maximum translation length $H$ so that if an $\eos$ action did not show up after $H$ steps, the environment will transit to the terminal state $S_{H+2}=(\X,\bos,y_1 \dots y_H, \eos)$ even if the agent continues to output normal token $A_{H+1}\in \vocab[\text{target}]$ at step $H+1$. When maximum translation length is applied, the corresponding episodic learning model of MT has bounded episode length, thus satisfies the ELP Condition (1) above. Such a formulation considers the mechanism of maximum translation length as a fundamental part of \emph{learning-based} MT task specification that is essential in helping \emph{learning} agents (which may not master when to output $\eos$) to escape from long and meaningless translation episodes which may otherwise lead to ill-conditioned training data. %Indeed, without translation length bound, even well-trained MT systems can still have a considerable chance to output ``hopelessly long'' translations.  

To comply with ELP Condition (3), we prescribe $\rho_0$ to be an arbitrary distribution over the set of terminal states, i.e. over all source-target sentence pairs $(X,Y)$ where $Y$ is complete sentence ending with $\eos$. As with other terminal steps, no matter what $S_0$ and $A_0$ are, the next state $S_1$ will follow the same distribution, denoted as $\rho_1$, which specifies the distribution of the source sentences that the agent will receive in \emph{each} episode (ELP Condition (2)).

Finally, the agent receives a scalar reward $R(X,Y) \in [0,100]$ at each terminal state, based on the translation quality of $Y$ (with respect to $X$). 
The reward is zero at all non-terminal states (which has only partial translations). 
\footnote{While it is certainly possible to make the rewards less sparse via reward shaping and engineering, we consider those reward variants as elements of a specific \emph{solution method}, instead of as part of the \emph{problem formulation} of MT.} 
With this reward function, the total-reward objective $J(\pi)$ of a translation policy $\pi$ corresponds to a sentence-averaged evaluation score over the corpus. For some MT metrics, this captures exactly the original metric; for example, the METEOR metric~\cite{2005:meteor} corresponds to $R(X,Y) = \texttt{METEOR}(Z(X),Y)$. For some other MT metrics, such as BLEU~\cite{2002:bleu}, the sentence-level BLEU score needs to be properly smoothed to match the true corpus-level BLEU score~\cite{2021:simpson}.

The ELP formulation of MT as discussed above can be formally summarized as follows:
%\squishlist
\begin{itemize}
\item $\SS = (\vocab[\text{source}])^H \times \{\bos\} \times (\vocab[\text{target}])^H \times \{\text{\textvisiblespace}~ , \eos\}$

\item $\AS = \vocab[\text{target}] \cup \{\eos\}$

%\item $R(s) \in [0,100]$ if $s\in\terminals$, otherwise $R(s) = 0$, where $\terminals = \{(X,Y): Y \text{ ends with } \eos\}$
\item $R(s) = \begin{cases}
\texttt{metric} \big(~ X(s),Y(s) ~\big) & s\in\terminals \\
0 & s\not\in\terminals
\end{cases}$
~,~ where $\terminals = \{(X,Y): Y \text{ ends with } \eos\}$

\item $\T(s'|s,a) = \begin{cases}
\rho_1(s') &  \text{~if~} s\in\terminals \\
\indicator{~s'=(s,a)~} & \text{~if~} s\not\in\terminals \text{~and~} |Y(s)|<H \\
\indicator{~s'=(s,\eos)~} & \text{~if~} s\not\in\terminals \text{~and~} |Y(s)|=H
\end{cases}$

\item $\rho_0(s) > 0$ only if $s\in\terminals$
\end{itemize}
%\squishend 
Note that in above both $\SS$ and $\AS$ are finite sets, in which case the model is a \emph{finite} episodic learning process. For real-world machine translation, we typically have $|\SS|<40000^{2048} \times 40000^{2048} \times 2$ and $|\AS|<40000+1$.

The experimentation code in Supplementary Material contains a faithful implementation in Python of the formulation presented here.

\subsection{Experiment Details}
\label{sec:exp_setup}
We tested our algorithmic idea using the WMT'14 NewsTest English$\ra$German (en2de) dataset~\footnote{https://nlp.stanford.edu/projects/nmt/}. The data was pre-processed and post-processed using the BPE tokenizer provided by YouTokenToMe~\footnote{https://github.com/VKCOM/YouTokenToMe}, with shared vocabulary of size $37000$. A complete translation typically consists of 20-100 tokens (meaning that a translation episodes contains roughly 20-100 action steps). We used SacreBLEU~\cite{2018:sacrebleu} to generate the BLEU scores, and trained the standard TransformerBase neural network~\cite{2017:transformer}, which is known to achieve a BLEU score of $27.3$ on the WMT'14 dataset under the state-of-the-art method of MLE-based supervised learning~\cite{2017:transformer}.

We trained the model on the same 4.5 millions sentence pairs in the WMT'14 data set for $100,000$ gradient updates on a V100 GPU, with the same mini-batch size (and token-padding strategy) and learning rate schedule as recommended by \citet{2017:transformer}. A dropout rate of $0.1$ is applied. The learned model is then used as search heuristic in the \emph{vanilla-beam-search decoding} (e.g. see Algorithm 1 in \cite{2019:stahlberg}), with a beam size of $4$. Empirically, we found that some more performance gain can be obtained by adding more tricks, such as modestly increasing the beam size (say, to $10$), adding length penalty factor in the search heuristic (see \cite{2016:gnmt}), and model averaging (see \cite{2017:transformer}), but we chose to exclude these tricks in our performance report so as to keep our algorithm simple and easy to implement. 

The following tables give the numerical values of the performance scores shown in Figure \ref{fig:temperature}.

\begin{table}[H]
	\centering
	\begin{tabular}{|c|c|}   		
	\hline 
	\rule[-1ex]{0pt}{4ex} $\beta$ & BLEU@100k \\ 
	\hline \rule{0pt}{2ex} 
	0.01	& 27.4 \\ 
	%\hline 
	0.25 	& 27.0 \\ 
	%\hline  
	0.5 	& 26.7 \\ 
	%\hline 
	0.75 	& 26.6 \\ 
	%\hline 
	1.0 	& 26.6 \\ 
	\hline 
	\end{tabular} 
	%\vspace{10pt}
	\caption{Corpus BLEU scores of LAMIN1 (i.e. Algorithm \ref{algo:lamin1}) under different temperature $\beta$.}
\label{tab:result_lamin1}
\end{table}

\begin{table}[H]
	\centering
	\begin{tabular}{|c|c|}   		
	\hline 
	\rule[-1ex]{0pt}{4ex} $\beta$ & BLEU@100k \\ 
	\hline \rule{0pt}{2ex} 
	0.01	& 26.0 \\ 
	%\hline 
	0.2 	& 26.2 \\ 
	%\hline  
	0.4 	& 26.8 \\ 
	%\hline 
	0.6 	& 26.2 \\ 
	%\hline 
	0.8 	& 26.3 \\ 
	%\hline 
	1.0 	& 26.8 \\ 
	\hline 
	\end{tabular} 
	%\vspace{10pt}
	\caption{Corpus BLEU scores of LAMIN2 (i.e. Algorithm \ref{algo:lamin2}) under different temperature $\beta$.}
\label{tab:result_lamin2}
\end{table}

%\section{Peer Reviews}
%\includepdf[pages=-]{nips21_review.pdf}
%\includepdf[pages=-]{iclr22_review.pdf}

\end{document}